\documentclass{article} %
\usepackage[final]{colm2026_conference}

\usepackage{microtype}
\usepackage{hyperref}
\usepackage{url}
\usepackage{booktabs}
\usepackage{amsmath}
\usepackage{amssymb}
\usepackage{cleveref}
\usepackage{amsthm}
\usepackage{algorithm}
\usepackage{algpseudocode}
\usepackage{color}
\usepackage{subcaption}
\usepackage{graphicx}
 
\usepackage{ifthen}
\newboolean{DisplayComments}
\setboolean{DisplayComments}{false}
\DeclareRobustCommand{\aj}[1]{\ifthenelse{\boolean{DisplayComments}}{{\color{red} (AJ: #1)}}{}}
\DeclareRobustCommand{\yc}[1]{\ifthenelse{\boolean{DisplayComments}}{{\color{blue} (YC: #1)}}{}}
\DeclareRobustCommand{\sn}[1]{\ifthenelse{\boolean{DisplayComments}}{{\color{violet} (SN: #1)}}{}}

\newboolean{ShowDeletedParagraphs}
\setboolean{ShowDeletedParagraphs}{true}
\setboolean{ShowDeletedParagraphs}{false}
\DeclareRobustCommand{\delete}[1]{\ifthenelse{\boolean{ShowDeletedParagraphs}}{{\color{teal} (DELETE START: )} #1 {\color{teal} (DELETE END) }}{}}

\newtheorem{theorem}{Theorem}

\newtheorem{lemma}[theorem]{Lemma}

\usepackage{lineno}

\definecolor{darkblue}{rgb}{0, 0, 0.5}
\hypersetup{colorlinks=true, citecolor=darkblue, linkcolor=darkblue, urlcolor=darkblue}

\title{Safe Inference-Time Alignment via Lagrangian Reward Augmentation}

\author{
  Yaswanth Chittepu$^{1 \thanks{Correspondence to \texttt{ychittepu@umass.edu}.\protect\\ Code: \url{https://github.com/yaswanthchittepu/LARA-Code}}}$, Ativ Joshi$^{1}$, Sohini Chintala$^{2}$, Scott Niekum$^{1}$ \\
  $^{1}$University of Massachusetts Amherst, 
  $^{2}$Independent Researcher\\
}

\begin{document}

\ifcolmsubmission
\linenumbers
\fi

\maketitle

\begin{abstract}
Inference-time alignment steers a frozen language model during decoding using auxiliary reward signals, avoiding the cost of repeated weight updates. However, existing inference-time alignment methods typically optimize a single scalar score, so explicit safety constraints must either be ignored or encoded through manually tuned penalties. We propose \emph{Lagrangian Reward Augmentation} (LARA), a general inference-time alignment framework under safety constraints. 
Starting from a KL-regularized constrained objective with a reward model and a cost model, LARA dualizes the constraint and reduces the optimization problem to a one-dimensional convex problem over a nonnegative dual variable. Estimated on a small calibration set, this dual variable defines an augmented reward that can be used as a drop-in scoring signal within existing inference-time alignment methods. For sequence-level sampling methods, such as Best-of-\(N\) reranking, the calibrated dual variable corresponds to the solution of the expected-cost constrained problem. For token-level reward-guided decoding methods, the same construction yields a principled dual-calibrated heuristic rather than an exact constrained-policy guarantee. We evaluate LARA on both sequence-level and token-level inference-time alignment methods, and find that LARA improves the helpfulness-harmlessness tradeoff, with Best-of-N achieving the best performance among inference-time methods, approaching finetuning-based direct alignment baselines.

\end{abstract}

\section{Introduction} 

Reinforcement Learning from Human Feedback (RLHF) \citep{ouyang2022training} has become a standard framework for aligning large language models with human preferences by first collecting preference data, then learning a reward model, and finally fine-tuning the model to prefer higher-reward outputs. In many deployments, however, usefulness alone is insufficient: modern assistants must also avoid harmful behavior, especially in sensitive domains such as medical consultation \citep{yang2022large, moor2023foundation}, legal reasoning \citep{katz2024gpt}, and educational support \citep{kasneci2023chatgpt, kung2023performance}. 
Safe RLHF \citep{dai2023safe}, motivated by the helpfulness and safety objectives conflict \citep{gehman2020realtoxicityprompts, weidinger2021ethical, ganguli2022red}, learns separate reward and cost models and optimizes under an explicit safety constraint.

However, traditional alignment still relies on policy fine-tuning, whether through the classical RLHF pipeline or through direct alignment methods such as DPO~\citep{rafailov2023direct} and related Direct Alignment Algorithms~\citep{rafailov2024scaling} that update model weights to satisfy preference data. Such retraining is expensive and may need to be repeated across applications, user populations, or changing safety requirements.
This cost has motivated a growing line of work on inference-time alignment, which leaves the base model frozen and instead steers decoding using auxiliary reward signals. Representative examples include sequence-level Best-of-N (BoN) reranking~\citep{stiennon2022learningsummarizehumanfeedback, nakano2022webgptbrowserassistedquestionansweringhuman, beirami2024theoretical} and token-level reward-guided search methods~\citep{khanov2024args, rashid2024critical, mudgal2023controlled, deng-raffel-2023-reward, rashid2025towards, han2024valueaugmentedsamplinglanguage}, all of which seek to obtain better aligned generations without running a full alignment stage, that would involve expensive model fine-tuning.

Despite this promise, existing inference-time alignment methods do not provide a satisfactory solution for \emph{safe} alignment. They require the practitioner to specify a scalar decoding score ahead of time, so when helpfulness and safety compete, the trade-off is typically hand-tuned rather than derived from a principled constrained objective. Sequence-level reranking methods can select high-scoring responses, but they do not by themselves enforce an explicit safety budget during selection. Token-level reward-guided methods can bias decoding online, but they still optimize a pre-specified scalar score rather than the solution of a constrained safe-alignment problem. As a result, current inference-time methods offer control, but not a principled mechanism for translating a safety constraint into a decoding objective.

We address this gap with \emph{Lagrangian Reward Augmentation} (LARA), an inference-time alignment framework that transfers the constrained trade-off of Safe RLHF~\citep{dai2023safe} to decoding. The key idea is to dualize the safety constraint, yielding a single nonnegative multiplier \(\lambda\) and the augmented reward
\(
r_\lambda(x,y) = r(x,y) - \lambda c(x,y).
\)
Calibrating \(\lambda\) on a small set of prompts produces a safety-aware score that can be plugged into existing inference-time alignment methods. For sequence-level samplers such as Best-of-\(N\)~\citep{stiennon2022learningsummarizehumanfeedback, nakano2022webgptbrowserassistedquestionansweringhuman, beirami2024theoretical}, this construction corresponds directly to the expected-cost constrained optimum. For token-level reward-guided decoders~\citep{khanov2024args, rashid2024critical, mudgal2023controlled, deng-raffel-2023-reward, rashid2025towards, han2024valueaugmentedsamplinglanguage}, it instead provides a principled dual-calibrated heuristic induced by the same constrained objective.

Our contributions are as follows. First, we formulate safe inference-time alignment through the dual of the Safe RLHF constrained objective, showing that the original policy optimization reduces to a one-dimensional convex calibration problem whose solution defines a safety-calibrated augmented reward. Second, we characterize the geometry of the dual objective, establishing monotonicity of the dual gradient and thereby yielding an efficient binary-search calibration procedure. Third, we prove a finite-sample guarantee for estimating the optimal dual variable from a small calibration dataset.
Fourth, we show how the resulting augmented reward can be used within a broad class of existing inference-time alignment methods, with an exact expected-cost interpretation for sequence-level samplers and a principled heuristic interpretation for token-level reward-guided decoders. Experiments show that LARA improves the helpfulness-harmlessness tradeoff, with Best-of-N achieving the best performance among inference-time methods, approaching finetuning-based direct alignment baselines

\section{Background}
\label{sec:background}

\subsection{Reinforcement Learning from Human Feedback}

Reinforcement Learning from Human Feedback (RLHF)~\citep{ouyang2022training} is the dominant paradigm for aligning language models to human preferences. The standard RLHF paradigm involves three stages \citep{ouyang2022training}. The first stage is Supervised Fine Tuning (SFT), where the model is trained on high-quality instruction data provided by human or LLM annotators, through the next-token prediction objective. The second stage is Reward Modeling (RM), wherein a reward model $r$ is trained to capture human preference. The reward model is trained on preference datasets using the Bradley-Terry preference model \citep{bradley1952rank}. In the final stage, the policy or language model is trained to produce responses preferred by humans through Reinforcement Learning (RL) using the learned reward model. In order to preserve language quality and prevent reward overoptimization\citep{ouyang2022training,stiennon2022learningsummarizehumanfeedback,Gao2022ScalingLF,rafailov2024scaling}, RLHF optimizes a KL regularized reward objective in the final stage. 

\begin{align}
  \max_\theta\;\mathbb{E}_{x \sim D_x,\, y \sim \pi_\theta(\cdot\mid x)}\!\big[r(x,y)\big]\;-\;\beta\,D_{\mathrm{KL}}\!\Big(\pi_\theta(\cdot\mid x)\,\Vert\,\pi_{\mathrm{ref}}(\cdot\mid x)\Big),
\end{align}

\subsection{Safe RLHF}

Standard RLHF optimizes a single reward function, which may be inadequate in the presence of competing objectives such as \emph{helpfulness} and \emph{harmlessness}. Safe RLHF separates these two signals, by learning a reward model $r_{\phi}(x,y)$ from helpfulness preference labels, and a cost model $c_{\psi}(x,y)$ from harmfulness preference labels and then frames learning as a constrained optimization problem.

\begin{align}
  \max_\theta\;\mathbb{E}_{x \sim D_x,\, y \sim \pi_\theta(\cdot\mid x)}\!\big[r(x,y)\big] - \beta D_{\mathrm{KL}}(\pi_\theta(\cdot\mid x)\Vert \pi_{\mathrm{ref}}(\cdot\mid x)) \quad\text{s.t.}\quad \mathbb{E}_{x \sim D_x,\, y \sim \pi_\theta(\cdot\mid x)}\!\big[c(x,y)\big]\le \tau.
  \label{eq:safe-rlhf-obj}
\end{align}

Follow-up works, such as HC-RLHF \citep{chittepu2025reinforcement} replace the expectation constraint with a high-confidence probabilistic safety guarantee using the Seldonian framework \citep{thomas2019preventing}. RAD \citep{chittepu2026safe} instead imposes a distributional dominance constraint, which they solve using optimal transport~\citep{peyre2020computationaloptimaltransport}. \citet{kim2025safedpo} propose SafeDPO, a direct alignment variant of Safe-RLHF that optimizes the safety-constrained RLHF problem using supervised learning. Unlike these methods, which modify the training objective and update model weights, LARA leaves the base model frozen and transfers the reward--cost trade-off to inference time through decoder-side calibration. %

\subsection{Inference-time Alignment}

Inference-time alignment seeks to avoid weight updates altogether by steering a frozen base model during decoding. At a high level, sequence-level methods~\citep{nakano2022webgptbrowserassistedquestionansweringhuman, stiennon2022learningsummarizehumanfeedback, beirami2024theoretical} sample several complete responses and select the highest scored response, by a reward model, while token-level reward-guided methods~\citep{khanov2024args, rashid2024critical, mudgal2023controlled, deng-raffel-2023-reward, rashid2025towards, han2024valueaugmentedsamplinglanguage} modify next-token probabilities using an auxiliary scoring function. For a reference policy \(\pi_{\mathrm{ref}}\), a prompt \(x\), and a partial response \(y_{1:t-1}\), a token-level reward-guided rule takes the form
\[
\mathrm{score}(y_t \mid x, y_{1:t-1})
=
\log \pi_{\mathrm{ref}}(y_t \mid x, y_{1:t-1}) + w \, r(x, y_{1:t}),
\]
which induces the tilted distribution
\[
\mathrm{softmax}\!\bigl(\mathrm{score}(y_t \mid x, y_{1:t-1})\bigr)
\propto
\pi_{\mathrm{ref}}(y_t \mid x, y_{1:t-1}) \exp\!\bigl(w\, r(x, y_{1:t})\bigr).
\]
Different inference-time alignment methods differ primarily in how this scoring signal is defined or approximated; we discuss these variants in Appendix~\ref{sec:related_work}. LARA is orthogonal to this decoder design choice: it calibrates the scalar trade-off between reward and cost and then supplies the resulting augmented score to any chosen inference-time procedure, thereby replacing manual scalar trade-off tuning with a principled safety-constrained construction.

\section{Method: LAgrangian Reward Augmentation (LARA)}
\label{sec:method}

Our goal is to perform safety-aware alignment at inference time without updating model
weights. We begin from the Safe RLHF objective, which separates helpfulness and safety
through a reward model and a cost model, and ask whether its constrained optimization
structure can be transferred from training time to decoding time. LAgrangian Reward Augmentation (LARA) does so by moving
the safety trade-off into a single scalar dual variable: once this variable is calibrated, it
induces an augmented reward that can be used by any standard inference-time alignment
procedure.

This section proceeds in two steps. Section~\ref{sec:model} presents the model, the dual formulation, and the empirical calibration procedure used to estimate the dual variable from a small calibration set. The output of this calibration step is a reward of the form
\[
r_\lambda(x,y)=r(x,y)-\lambda  c(x,y),
\]
which can be passed as a drop-in scoring signal to any existing inference-time decoding algorithm.
Thus, LARA should be viewed as a decoder-agnostic calibration layer rather than as a new
decoding algorithm.
Section~\ref{sec:method_estimation} then gives the theoretical guarantee for this estimator under finite samples.

\subsection{Model and Estimation}
\label{sec:model}

Our starting point is the Safe RLHF formulation, which separates helpfulness and safety through
a reward model $r(x,y)$ and a cost model $c(x,y)$, where $x$ is a prompt and $y$ is
a response. This subsection has three steps. We first state the constrained objective and explain why strong duality allows us to work with its dual. We then derive the resulting one-dimensional dual objective and the corresponding Gibbs-tilted policy for a fixed multiplier \(\lambda\). Finally, we describe how \(\lambda\) is estimated empirically from a small calibration set and how the resulting augmented reward is passed to an existing inference-time decoder.

We assume throughout that the reward and cost models are fixed and bounded:
there exist $R,C>0$ such that $|r(x,y)| \le R$ and $|c(x,y)| \le C$ for all $(x,y)$.
Let $D_x$ denote the prompt distribution, let $\pi_{\mathrm{ref}}(\cdot \mid x)$ denote the reference
distribution over responses for prompt $x$, let $\beta>0$ be the KL regularization coefficient,
and let $\tau \in [-C,C]$ 
be the safety budget. We consider the constrained optimization problem
\begin{equation}
\label{eq:primal}
\sup_{\pi}\;
\mathbb{E}_{x\sim D_x}
\left[
\mathbb{E}_{y\sim \pi(\cdot \mid x)} r_{\phi}(x,y)
-
\beta D_{\mathrm{KL}}\!\left(\pi(\cdot \mid x)\,\|\,\pi_{\mathrm{ref}}(\cdot \mid x)\right)
\right]
\text{ s.t. }
\mathbb{E}_{x\sim D_x,\; y\sim \pi(\cdot \mid x)} c_{\psi}(x,y) \le \tau,
\end{equation}
where the optimization is over conditional distributions $\pi(\cdot\mid x)$ that are absolutely
continuous with respect to $\pi_{\mathrm{ref}}(\cdot\mid x)$ for $D_x$-almost every $x$, so that the
conditional KL divergence is well defined. 

The optimization problem in \eqref{eq:primal} is a concave maximization problem in the policy variable \(\pi\): the objective is concave in \(\pi\), and the expected-cost constraint is affine. We assume a Slater condition, namely that there exists a feasible policy \(\tilde{\pi}\) whose expected cost is strictly below \(\tau\). Under this regularity condition, strong duality holds. Consequently, solving the primal problem is equivalent to minimizing the dual function over \(\lambda \ge 0\).

The key observation behind LARA is that, after dualizing the safety constraint, the
infinite-dimensional optimization over policies reduces to a one-dimensional optimization over
a scalar dual variable. For $\lambda \ge 0$, define the net reward
\begin{equation}
\label{eq:net-reward}
r_{\lambda}(x,y) := r_{\phi}(x,y) - \lambda c_{\psi}(x,y),
\end{equation}
and the corresponding Lagrangian
\begin{equation}
\label{eq:lagrangian}
\mathcal L(\pi,\lambda)
:=
\mathbb{E}_{x\sim D_x}
\left[
\mathbb{E}_{y\sim \pi(\cdot \mid x)} r_{\lambda}(x,y)
-
\beta D_{\mathrm{KL}}\!\left(\pi(\cdot \mid x)\,\|\,\pi_{\mathrm{ref}}(\cdot \mid x)\right)
\right]
+
\lambda \tau.
\end{equation}
The associated dual function is
\begin{equation}
\label{eq:dual-function}
g(\lambda) := \sup_{\pi} L(\pi,\lambda),  \qquad \lambda \ge 0,
\end{equation}
where the supremum is taken over the same class of policies as in the primal problem. The dual problem is then to minimize $g$ over $\lambda \ge 0$:
\begin{equation}
\label{eq:dual-problem}
\inf_{\lambda \ge 0} g(\lambda).
\end{equation}
The next theorem shows that this dual objective admits a closed form, and that for a fixed
value of $\lambda$, the optimal policy is a Gibbs tilt of the reference model by the net reward
$r_\lambda$. 
The derivation is similar in spirit to the KL-regularized optimal-policy characterization used by \cite{rafailov2023direct}, we provide the full proof in Appendix~\ref{sec:appendix-duality} for completeness.

\begin{theorem}[Closed-form dual objective]
\label{thm:dual-closed-form}
For every $\lambda \ge 0$,
\begin{equation}
\label{eq:dual-objective}
g(\lambda)
=
\beta\,
\mathbb{E}_{x\sim D_x}
\left[
\log
\mathbb{E}_{y\sim \pi_{\mathrm{ref}}(\cdot \mid x)}
\exp\!\left(
\frac{r_{\phi}(x,y)-\lambda c_{\psi}(x,y)}{\beta}
\right)
\right]
+
\lambda \tau.
\end{equation}
Moreover, for each prompt $x$, the supremum in equation~\eqref{eq:dual-function} is attained by the Gibbs-tilted policy %
\begin{equation}
\label{eq:gibbs-policy}
\pi_{\lambda}^{\star}(y\mid x)
=
\frac{
\pi_{\mathrm{ref}}(y\mid x)\,
\exp\!\left(\frac{r_{\lambda}(x,y)}{\beta}\right)
}{
\mathbb{E}_{y'\sim \pi_{\mathrm{ref}}(\cdot \mid x)}
\exp\!\left(\frac{r_{\lambda}(x,y')}{\beta}\right)
}.
\end{equation}
\end{theorem}
Theorem~\ref{thm:dual-closed-form} is the main structural step in our method. Equation~\eqref{eq:dual-objective} replaces the original
optimization over policies by a one-dimensional minimization over $\lambda$, while
Equation~\eqref{eq:gibbs-policy} shows that, for a fixed dual variable, the optimal policy is obtained by tilting
the reference model toward responses with high net reward. This is precisely the object that
LARA will later approximate. %

Clearly, for \(\lambda=0\), the dual objective reduces to the standard KL-regularized reward maximization objective, and the Gibbs-tilted policy reduces to the standard reward-only tilt. As \(\lambda\) increases, the dual objective incorporates an increasing penalty on harmfulness, and the Gibbs-tilted policy incorporates an increasing penalty on responses with high cost.
The optimal dual variable \(\lambda^{\star}\) balances these competing objectives to achieve the best helpfulness subject to the harmlessness budget. %

We next characterize the geometry of the dual objective, see Appendix~\ref{sec:lemma-proof} for the proof.
\begin{lemma}[Derivative and Hessian of the dual objective]
\label{lem:dual-derivatives}
The first and second derivatives of $g$ are
\begin{equation}
\label{eq:dual-derivative}
g'(\lambda)
=
\tau
-
\mathbb{E}_{x\sim D_x,\; y\sim \pi_{\lambda}^{\star}(\cdot \mid x)}
\big[c_{\psi}(x,y)\big],
\end{equation}
and
\begin{equation}
\label{eq:dual-hessian}
g''(\lambda)
=
\frac{1}{\beta}\,
\mathbb{E}_{x\sim D_x}
\left[
\operatorname{Var}_{y\sim \pi_{\lambda}^{\star}(\cdot \mid x)}
\big(c_{\psi}(x,y)\big)
\right]
\ge 0.
\end{equation}
In particular, $g$ is convex on $[0,\infty)$. 
Moreover, if we restrict $\lambda$ to a compact
interval $I:=[0,\Lambda]$ (for some $\Lambda < \infty$) and the variance term in \eqref{eq:dual-hessian} is continuous and bounded away
from zero on $I$, then there exists $\mu>0$ such that $g''(\lambda)\ge \mu$ for all
$\lambda\in I$, so $g$ is strongly convex on $I$.
\end{lemma}

Lemma~2 yields the main interpretation of the dual variable. Equation~\eqref{eq:dual-derivative} shows that the derivative is exactly the harmlessness budget minus the expected cost under the Gibbs-tilted policy. Hence any interior minimizer $\lambda^\star > 0$ satisfies
\[
\mathbb{E}_{x \sim D_x,\; y \sim \pi^\star_{\lambda^\star}(\cdot \mid x)}[c(x,y)] = \tau,
\]
so $\lambda^\star$ has the usual complementary-slackness interpretation as a shadow price for harmfulness. Equation~\eqref{eq:dual-hessian} further shows that $g$ is convex, so $g'$ is monotone nondecreasing on $[0,\infty)$; equivalently, the map
\[
\lambda \mapsto \mathbb{E}_{x \sim D_x,\; y \sim \pi^\star_\lambda(\cdot \mid x)}[c(x,y)]
\]
is monotone nonincreasing. This monotonicity is the key computational advantage of the dual reduction. On a compact search interval $I=[0,\Lambda]$, if $g'(0)\ge 0$, then the minimizer is the boundary point $\lambda^\star=0$; otherwise any interior minimizer is characterized by the scalar root condition $g'(\lambda^\star)=0$ and can therefore be found efficiently by binary search. If, in addition, $g''(\lambda)\ge \mu > 0$ on $I$, then $g$ is strongly convex on $I$, the minimizer is unique, and this one-dimensional search problem is especially well conditioned. Hence, we therefore estimate $\lambda^\star$ over a compact interval $I=[0,\Lambda]$, which turns calibration into a stable one-dimensional search problem. Specifically, we restrict the search to a compact interval because on the unbounded domain  $[0,\infty)$  the minimizer need not be attained at a finite  $\lambda$, and a uniform strong-convexity bound need not hold. The compact restriction localizes the problem to a well-conditioned one-dimensional search region.

Fix a compact interval $I := [0,\Lambda] \subset \mathbb{R}_+$ and let $\beta>0$.
For each prompt $x$, define
\begin{equation}\label{eq:thm_empirical_inner_Z_phi}
Z_x(\lambda)
:=
\mathbb{E}_{Y\sim \pi_{\mathrm{ref}}(\cdot\mid x)}
\exp\!\left(\frac{r(x,Y)-\lambda c(x,Y)}{\beta}\right),
\qquad
\phi(x,\lambda):=\beta \log Z_x(\lambda)+\lambda\tau,
\end{equation}
so that 
\begin{equation}\label{eq:thm_empirical_inner_g}
g(\lambda):=\mathbb{E}_{X\sim D_x}[\phi(X,\lambda)],
\qquad
\lambda^\star \in \arg\min_{\lambda\in I} g(\lambda).
\end{equation}

In practice, we calibrate $\lambda$ using a small calibration dataset of prompts and multiple samples from the reference policy for each prompt. For a candidate value of $\lambda$, we approximate the expectations in~\eqref{eq:thm_empirical_inner_Z_phi} and~\eqref{eq:thm_empirical_inner_g} by Monte Carlo, average the resulting empirical objective over calibration prompts, and then optimize over $\lambda \in [0,\Lambda]$. By Lemma~2, the corresponding expected cost is monotone in $\lambda$, so this calibration reduces to a stable one-dimensional search problem. %

Now suppose $X_1,\dots,X_N \stackrel{\mathrm{i.i.d.}}{\sim} D_x$, and that for each
$i\in\{1,\dots,N\}$, conditional on $X_i$, the responses
$Y_{i,1},\dots,Y_{i,K} \stackrel{\mathrm{i.i.d.}}{\sim} \pi_{\mathrm{ref}}(\cdot\mid X_i)$, with the
tuples $(X_i,Y_{i,1},\dots,Y_{i,K})$ i.i.d. across $i$. We estimate the inner expectation by
\begin{equation}\label{eq:thm_empirical_inner_hatZ}
\widehat Z_{i,K}(\lambda)
:=
\frac{1}{K}\sum_{k=1}^K
\exp\!\left(\frac{r(X_i,Y_{i,k})-\lambda c(X_i,Y_{i,k})}{\beta}\right),
\end{equation}
and define the corresponding empirical inner loss and empirical dual objective as
\begin{equation}\label{eq:thm_empirical_inner_hatphi_hatg}
\widehat\phi_{i,K}(\lambda)
:=
\beta \log \widehat Z_{i,K}(\lambda)+\lambda\tau,
\qquad
\widehat g_{N,K}(\lambda)
:=
\frac{1}{N}\sum_{i=1}^N \widehat\phi_{i,K}(\lambda),
\end{equation}
An empirical calibration of the dual variable is then obtained by
\begin{equation}\label{eq:thm_empirical_inner_hatlambda}
\widehat\lambda_{N,K}\in \arg\min_{\lambda\in I}\widehat g_{N,K}(\lambda).
\end{equation}
Since $\hat g_{N,K}$ is convex in $\lambda$, its derivative is monotone nondecreasing on $[0,\Lambda]$. Therefore, when the minimizer is interior, $\hat\lambda_{N,K}$ can be computed efficiently by binary search for the root of $\hat g'_{N,K}(\lambda)=0$; otherwise, the solution lies at the boundary.

Once $\widehat\lambda_{N,K}$ has been computed, we define the calibrated augmented reward
\[
r_{\widehat\lambda_{N,K}}(x,y)
:=
r(x,y)-\widehat\lambda_{N,K}c(x,y),
\]
and pass this score to the chosen inference-time decoding procedure. In this sense, LARA calibrates a
scalar trade-off parameter rather than introducing a new decoding rule: the decoder remains unchanged, and only its reward signal is replaced by the safety-calibrated augmented reward. 
Algorithm~\ref{alg:lara} summarizes the procedure formally.

Note that the dual objective in Theorem~\ref{thm:dual-closed-form} and the Gibbs policy \(\pi^\star_\lambda\) are defined at the level of complete responses, so the exact expected-cost interpretation of \(\lambda^\star\) applies directly to sequence-level samplers such as Best-of-\(N\) reranking. When the calibrated augmented reward is instantiated inside a token-level reward-guided decoder, the resulting procedure need not sample exactly from \(\pi^\star_\lambda\) over full responses. In that setting, \(\hat{\lambda}_{N,K}\) should be interpreted as a principled dual-calibrated trade-off parameter rather than as a quantity that exactly enforces the population expected-cost constraint. We emphasize that this heuristic status is not specific to LARA. Deriving a token-level policy from the sequence-level Gibbs policy $\pi^\star_\lambda$ requires marginalizing over all future tokens, yielding an intractable next-token conditional that token-level reward-guided methods do not match \citep{rashid2024critical}. LARA introduces no approximation beyond the one these methods already commit to; it calibrates the trade-off parameter within that same approximation.

\begin{algorithm}[t]
\caption{LARA: Inference-Time Safe Alignment via Dual Calibration}
\label{alg:lara}
\begin{algorithmic}[1]
\Require Reference model $\pi_{\mathrm{ref}}$, reward model $r_{\phi}$, cost model $c_{\psi}$, harmlessness budget $\tau$, KL coefficient $\beta$, calibration prompts $D_{\mathrm{cal}}=\{X_i\}_{i=1}^{N}$, search interval $I=[0,\Lambda]$, inference-time decoder algorithm $\mathsf{DecodeAlg}$
\State For each $i\in\{1,\dots,N\}$, sample $Y_{i,1},\dots,Y_{i,K}
   \stackrel{\mathrm{i.i.d.}}{\sim} \pi_{\mathrm{ref}}(\cdot\mid X_i)$ and define
   \[
   \widehat g_{N,K}(\lambda)
   =
   \frac{1}{N}\sum_{i=1}^N
   \left[
   \beta \log \left(
   \frac{1}{K}\sum_{k=1}^K
   \exp\!\left(
   \frac{r(X_i,Y_{i,k})-\lambda c(X_i,Y_{i,k})}{\beta}
   \right)
   \right)
   + \lambda\tau
   \right].
   \]
\State Compute
   \(
   \widehat\lambda_{N,K}\in\arg\min_{\lambda\in[0,\Lambda]}\widehat g_{N,K}(\lambda)
   \) using binary search. 
\State Define the calibrated augmented reward
   \(
   r_{\widehat\lambda_{N,K}}(x,y):=r(x,y)-\widehat\lambda_{N,K}c(x,y)
   \) and pass it to $\mathsf{DecodeAlg}$. 
\end{algorithmic}
\end{algorithm}

Accordingly, LARA should be viewed as a decoder-agnostic layer that supplies a principled helpfulness--harmlessness trade-off to an existing inference-time alignment method. Its safety interpretation is the expected-cost, soft-constraint interpretation inherited from Safe RLHF: it calibrates a shadow price for harmfulness and then decodes with the resulting augmented reward. %

 In practice, calibration may be sensitive to mismatch between the calibration prompts and the deployment distribution. As a robustness heuristic, one may therefore bootstrap the calibration set to obtain a distribution over \(\hat{\lambda}_{N,K}\) and deploy an upper quantile of this distribution as a conservative choice of the trade-off parameter. We do not analyze this procedure theoretically here, so we view it as a practical robustness heuristic rather than part of the core guarantee.

\subsection{Theoretical Results}
\label{sec:method_estimation}

We now turn to the finite-sample behavior of the empirical calibration procedure defined in Section~\ref{sec:model}.  The following theorem shows that, under boundedness and strong convexity assumptions, the estimator \(\hat{\lambda}_{N,K}\) concentrates around the population minimizer \(\lambda^\star\), with error controlled by the outer-sample term in \(N\) and the inner Monte Carlo term in \(K\). %
\begin{theorem}[Finite-sample guarantee]
\label{thm:finite-sample-empirical}
Assume that $|r_(x,y)|\le R$ and $ |c_(x,y)|\le C$ for all $(x,y)$, that
$g$ has a unique minimizer $\lambda^\star\in(0,\Lambda)$, and that $g$ is
$\mu$-strongly convex on $I=[0,\Lambda]$, i.e., $g''(\lambda)\ge \mu$ for all
$\lambda\in I$. Define
\begin{equation}\label{eq:thm_empirical_inner_ab}
a:=\exp\!\left(-\frac{R+\Lambda C}{\beta}\right),
\qquad
b:=\exp\!\left(\frac{R+\Lambda C}{\beta}\right),
\end{equation}
\begin{equation}\label{eq:thm_empirical_inner_BK}
B_K
:=
\frac{Cb}{a\sqrt{K}}
+
\frac{Cb^2}{2a^2\sqrt{K}},
\end{equation}
\begin{equation}\label{eq:thm_empirical_inner_epsilon}
\varepsilon_{N,K}(\delta)
:=
2C\sqrt{\frac{2\log(2(N+1)/\delta)}{N}}
+
\frac{C^2\Lambda}{\beta N}
+
B_K,
\end{equation}
and
\begin{equation}\label{eq:thm_empirical_inner_m}
m:=\min\{\lambda^\star,\Lambda-\lambda^\star\}.
\end{equation}
If $\varepsilon_{N,K}(\delta)<\mu m$, then
\begin{equation}\label{eq:thm_empirical_inner_conclusion}
\mathbb{P}\!\left(
\bigl|\widehat\lambda_{N,K}-\lambda^\star\bigr|
\le
\frac{\varepsilon_{N,K}(\delta)}{\mu}
\right)
\ge 1-\delta.
\end{equation}
\end{theorem}

Theorem~\ref{thm:finite-sample-empirical} makes the calibration error transparent. The first term in
$\varepsilon_{N,K}(\delta)$ is the concentration error from using $N$ prompts, the second is a
one-dimensional discretization term arising in the uniform derivative bound, and the third is
the $O(K^{-1/2})$ Monte Carlo error from approximating the inner expectation. Thus, for
fixed $K$, the estimator is accurate up to the residual inner-expectation error, while letting
both $N$ and $K$ grow drives $\widehat\lambda_{N,K}$ toward the population minimizer
$\lambda^\star$.
See Appendix~\ref{app:duality-conc} for the proof.

\section{Experiments}

We investigate the following two research questions:
\begin{enumerate}
\item For Best-of-$N$ sampling, does LARA identify a principled $\lambda^*$ such that any smaller value violates the safety constraint and any larger value satisfies it, at unnecessary cost to helpfulness?
\item How does LARA-driven inference-time alignment compare to (i) finetuning-based baselines and (ii) inference-time baselines that either forgo safety penalties ($\lambda=0$) or over-penalize harmlessness ($\lambda > \lambda^*$)?
\end{enumerate}

We evaluate our research questions on two models: Llama-3.2-3B \citep{grattafiori2024llama3herdmodels} and Qwen-2.5-3B \citep{qwen2025qwen25technicalreport}. We follow the standard RLHF pipeline. The models are first finetuned on the Alpaca dataset \citep{alpaca} to obtain the SFT policy, which serves as the reference policy for all methods. The reward model and cost model are trained using the standard Bradley-Terry preference modeling objective on the BeaverTails preference dataset \citep{ji2023beavertails}, using helpfulness and harmfulness preference labels respectively. We note that this protocol follows that of Safe-RLHF \citep{dai2023safe} and SafeDPO \citep{kim2025safedpo}, allowing for a principled comparison.
For additional implementation details and hyperparameters, refer to Appendix~\ref{ap:implementation-details}.

\subsection{Constraint Satisfaction and Optimality of $\lambda^*$}

We evaluate RQ1 using Best-of-$N$ sampling on the BeaverTails test set \citep{ji2023beavertails}, which was also used for evaluations in Safe-RLHF \citep{dai2023safe} and Safe-DPO \citep{kim2025safedpo}, for which LARA is theoretically grounded. Figure~\ref{fig:bon_constraint} reports the expected helpfulness $\mathbb{E}[r]$ and expected cost (Harmfulness) 
$\mathbb{E}[c]$ of the BoN policy, which selects responses according to $r - \lambda c$, 
as a function of $\lambda$, for both Llama and Qwen. We observe that for $\lambda < 
\lambda^*$, the cost constraint $\mathbb{E}[c] \leq \tau$ is violated, while for $\lambda 
> \lambda^*$, the constraint is satisfied but at an unnecessary reduction in expected 
helpfulness. LARA's calibrated $\lambda^*$ sits at the boundary, satisfying the constraint 
while maximizing helpfulness, consistent with the theoretical optimality conditions derived 
in Section~\ref{sec:model}. To confirm that this calibration behavior is not specific to the model scale used in our main experiments, we reproduce the constraint-satisfaction analysis of Figure~\ref{fig:bon_constraint} for Llama-3.1-8B in Appendix~\ref{ap:model-scale} in Figure~\ref{fig:bon_8b}; the same pattern holds, with $\lambda^*$ satisfying the safety constraint.

\begin{figure}[t]
    \centering
    \begin{subfigure}[b]{0.48\textwidth}
        \centering
        \includegraphics[width=\textwidth]{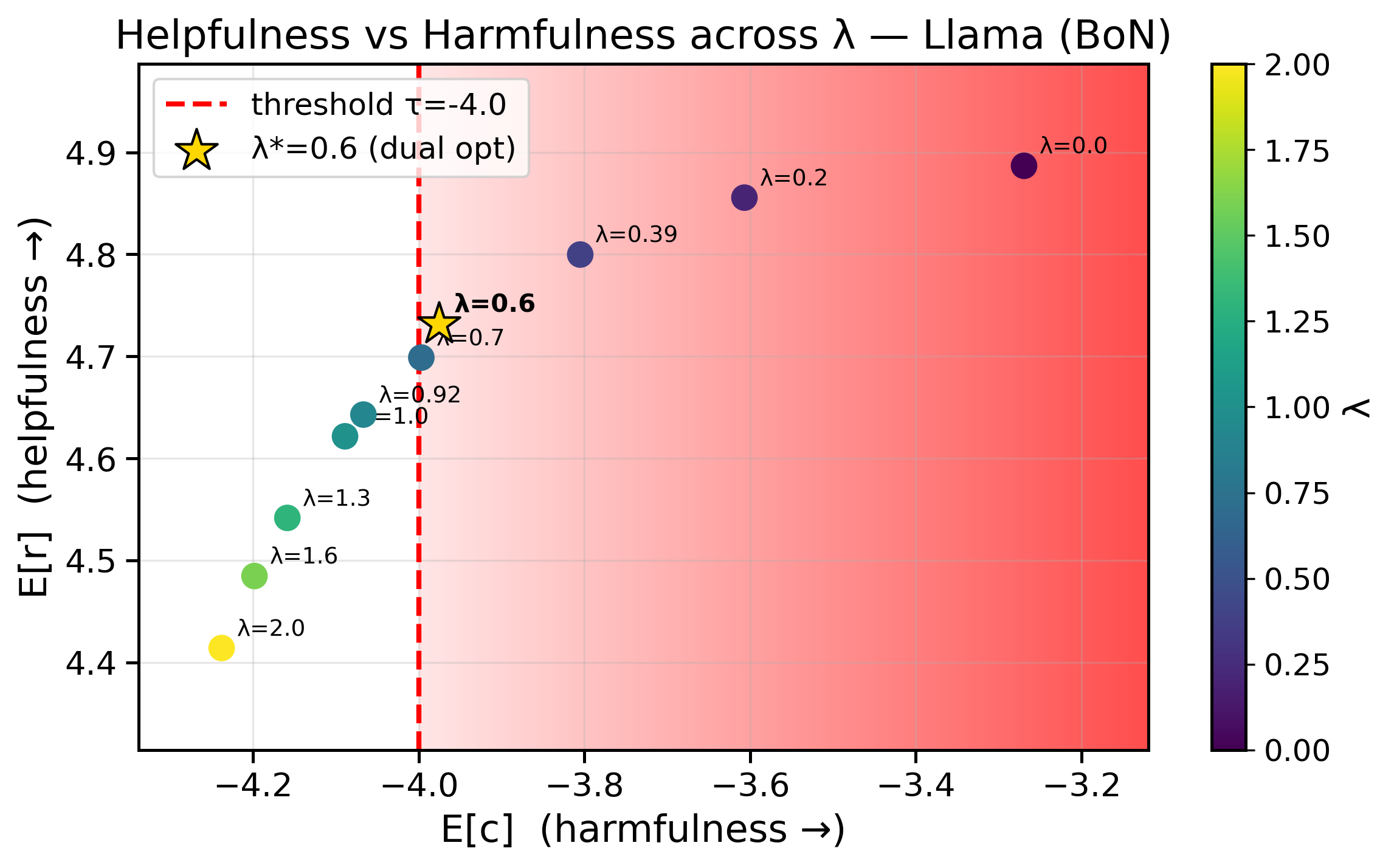}
        \label{fig:bon_llama}
    \end{subfigure}
    \hfill
    \begin{subfigure}[b]{0.48\textwidth}
        \centering
        \includegraphics[width=\textwidth]{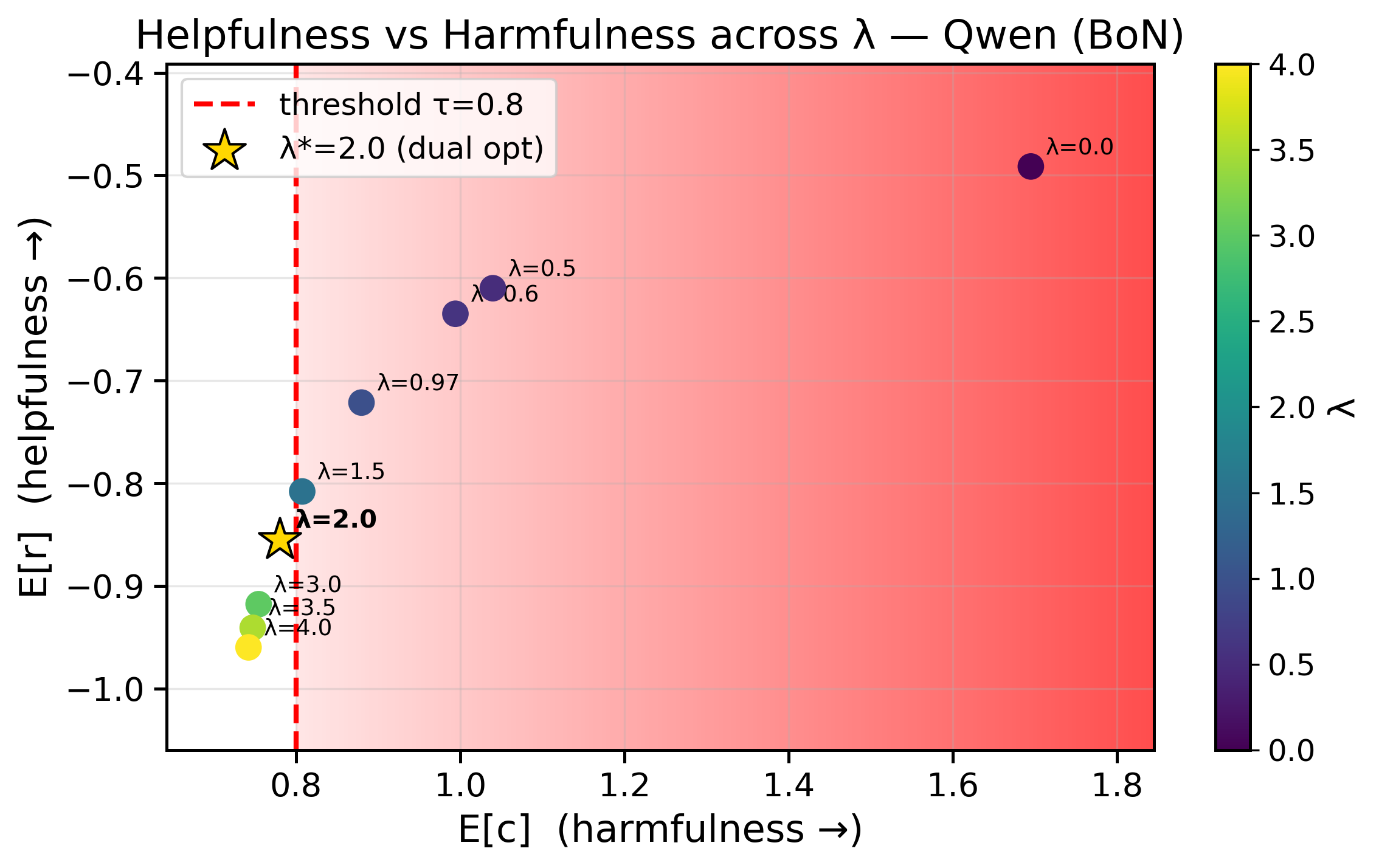}
        \label{fig:bon_qwen}
    \end{subfigure}
    \caption{Expected helpfulness $\mathbb{E}[r]$ and expected cost $\mathbb{E}[c]$ of the BoN policy as a function of $\lambda$, for the Llama base model on the left and Qwen base model on the right. The dashed vertical line indicates the expected cost threshold $\tau$.}
    \label{fig:bon_constraint}
\end{figure}

\subsection{Comparison Against Finetuning and Inference-Time Baselines}

We evaluate RQ2 by comparing LARA-driven inference-time alignment against 
finetuning-based baselines (Safe-RLHF, SafeDPO, SFT) and inference-time 
baselines at $\lambda=0$ (no safety penalty) and $\lambda > \lambda^*$ 
(over-penalized). We perform both model-based evaluation and evaluation with GPT as the judge \citep{zheng2023judgingllmasajudgemtbenchchatbot}. We employ the same evaluation protocol used in \citet{kim2025safedpo}. We provide results for Llama here and push those for Qwen to Appendix ~\ref{ap:additional-results}.

\paragraph{Model Evaluation}
We use the beaver-7b-unified-reward model\footnote{\url{https://huggingface.co/PKU-Alignment/beaver-7b-unified-reward}} 
to assess helpfulness and the beaver-7b-unified-cost model\footnote{\url{https://huggingface.co/PKU-Alignment/beaver-7b-unified-cost}} 
to assess harmlessness, over a set of 1024 test prompts from the BeaverTails \citep{ji2023beavertails} test set. 

Looking at model-based evaluation for Llama in Figure~\ref{fig:model_eval_llama}, BoN 
with LARA's calibrated $\lambda^*$ achieves the best helpfulness-harmlessness tradeoff 
among inference-time methods, outperforming SafeDPO on helpfulness while remaining 
competitive on harmlessness, despite requiring no weight updates. At $\lambda=0$, BoN 
produces highly harmful responses, while $\lambda > \lambda^*$ over-penalizes helpfulness 
with marginal harmlessness gain, validating the optimality of LARA's calibrated $\lambda^*$. 
All RGTG methods perform particularly poorly on helpfulness, falling below the SFT model, 
though they achieve higher harmlessness. Among RGTG methods, CD-Q is the most harmless, 
with harmlessness monotonically increasing with $\lambda$. Among all methods considered, 
Safe-RLHF achieves the best helpfulness-harmlessness tradeoff.

\begin{figure}[t]
    \centering
    \begin{subfigure}[b]{0.48\textwidth}
        \centering
        \includegraphics[width=\textwidth]{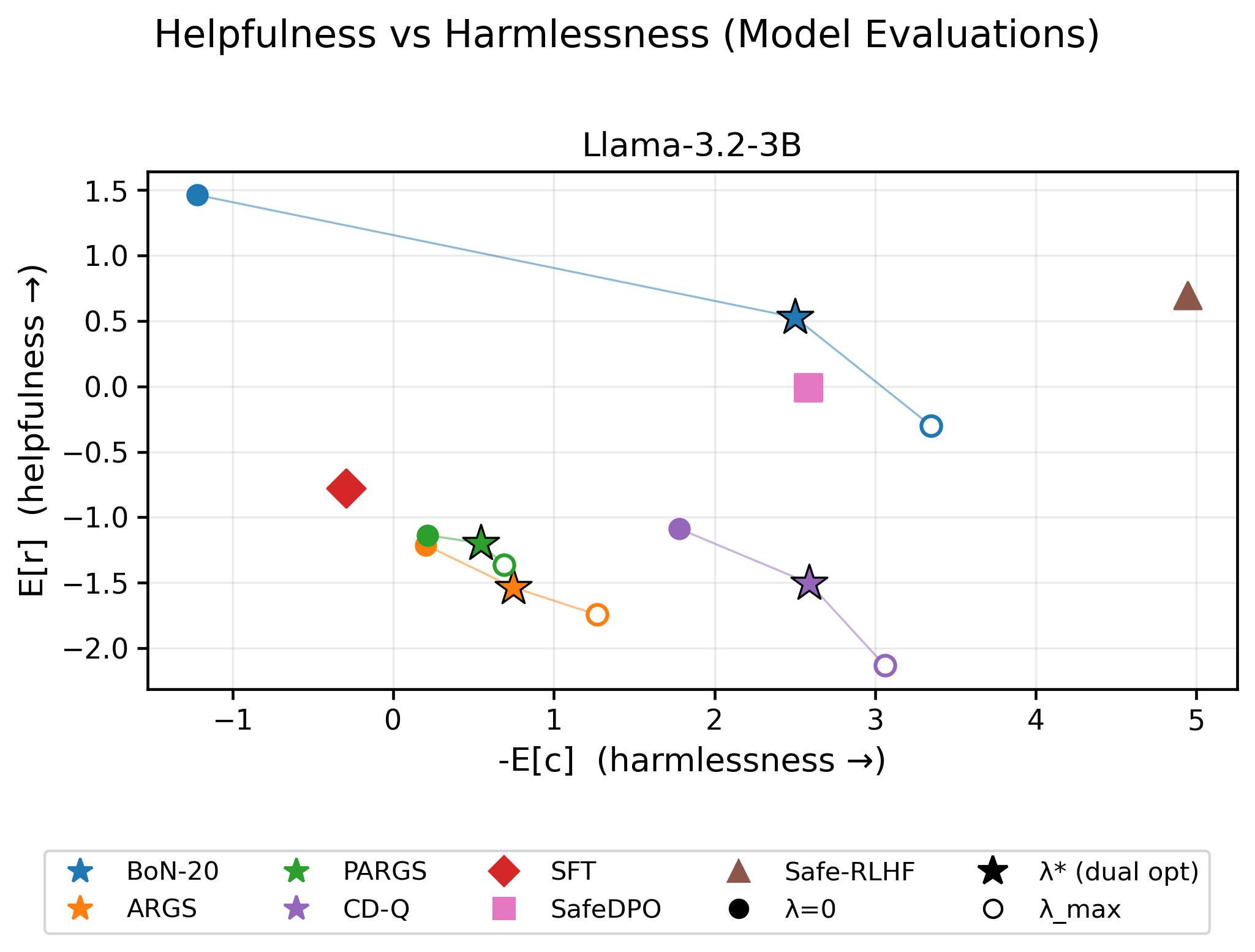}
        \caption{Model Evaluation}
        \label{fig:model_eval_llama}
    \end{subfigure}
    \hfill
    \begin{subfigure}[b]{0.48\textwidth}
        \centering
        \includegraphics[width=\textwidth]{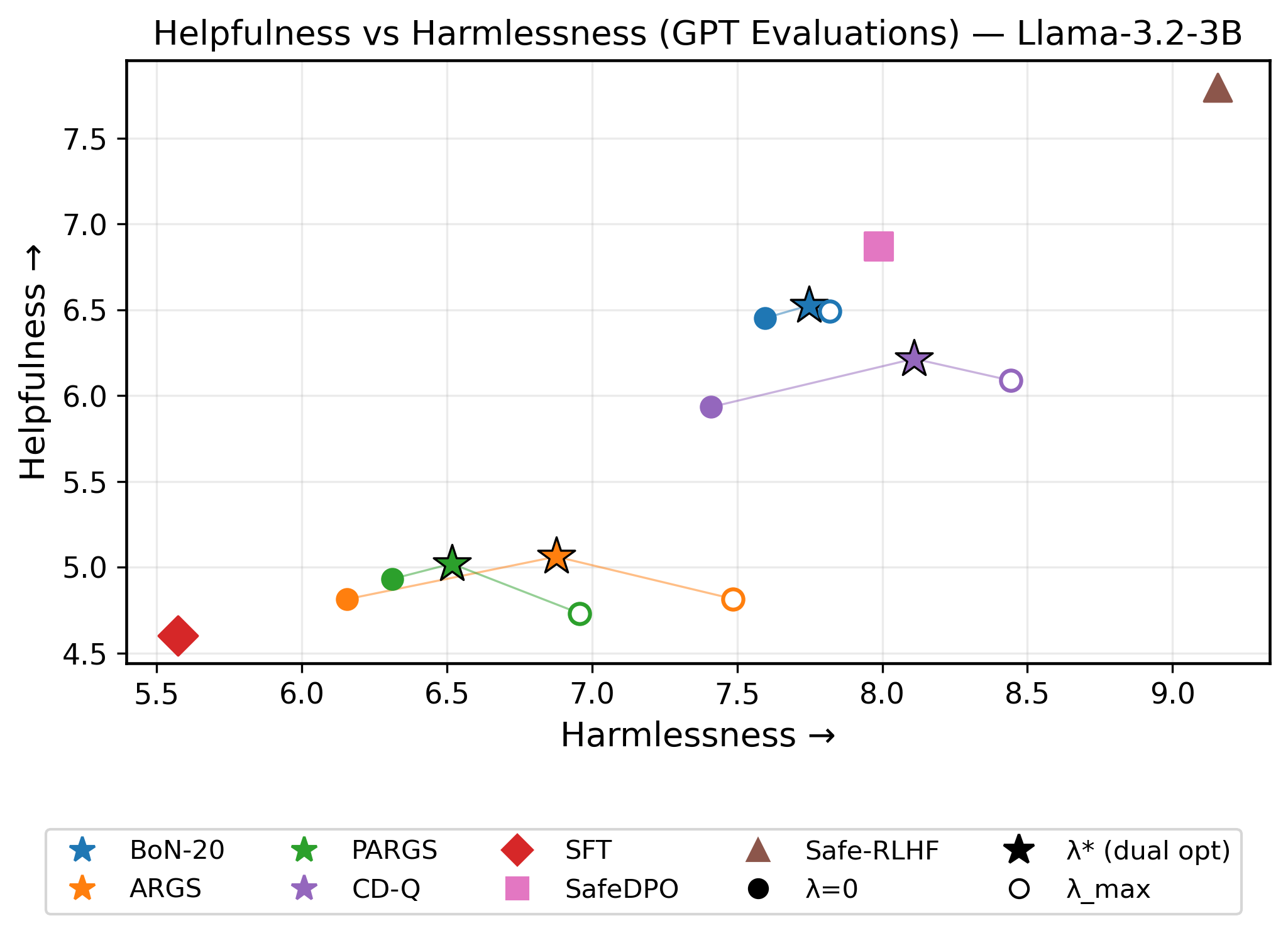}
        \caption{GPT Evaluation}
        \label{fig:gpt_eval_llama}
    \end{subfigure}
    \caption{Helpfulness vs Harmfullness for LLama3.2-3b under both model evaluation (left) and GPT evaluation (right). $\lambda^*$ for Llama is $0.6$ and $\lambda_{max}$ is chosen to be $2.0$.}
    \label{fig:eval_llama}
\end{figure}

\paragraph{GPT Evaluation} We used GPT-4o-mini as a judge\citep{zheng2023judgingllmasajudgemtbenchchatbot}, and evaluated the helpfulness and harmlessness of the model generations, over the same set of 1024 test prompts from the BeaverTails \citep{ji2023beavertails} test set, which were also used in model evaluations. We used the same evaluation prompts used by \citet{kim2025safedpo}, where GPT is asked to rate the helpfulness and harmlessness of model responses on a scale from 1-10.

Under GPT-based evaluation for Llama in Figure~\ref{fig:gpt_eval_llama}, the rankings 
are broadly consistent with model-based evaluation, with Safe-RLHF achieving the best 
overall tradeoff and finetuning baselines outperforming inference-time methods. BoN achieves comparable performance to SafeDPO. However, 
BoN shows minimal sensitivity to $\lambda$ under GPT evaluation, with helpfulness and 
harmlessness remaining largely stable across $\lambda=0$, $\lambda^*$, and $\lambda > 
\lambda^*$, suggesting that the beaver reward and cost models are more sensitive to 
$\lambda$ than GPT's judgment. According to GPT evaluation, CD-Q performs much better with harmlessness improving consistently with $\lambda$, and emerges as the strongest 
RGTG method. ARGS and PARGS continue to perform poorly on helpfulness, though harmlessness 
does improve with $\lambda$ for both methods.

To understand the discrepancy between model and GPT evaluation, we compute the Pearson correlation between the helpfulness and harmlessness ratings across all methods and $\lambda$ values. This correlation is 0.73 for GPT but only 0.13 for the gold reward and cost models, indicating that GPT tends to collapse helpfulness and harmlessness into a single quality dimension, whereas the gold models treat them as largely independent signals. Because LARA explicitly trades one dimension off against the other, an evaluator that conflates them is expected to show weak sensitivity to $\lambda$. Hence the flatter GPT response to $\lambda$ is consistent with, rather than evidence against, LARA's behavior.
We therefore recommend the gold model-based evaluation for assessing constraint satisfaction and the helpfulness-harmlessness tradeoff, since those models score the two axes independently, precisely the separation LARA is designed to exploit, with GPT evaluation serving as a complementary signal for overall response quality.

\paragraph{Why CD-Q Outperforms ARGS and PARGS}

Among the token-level methods, ARGS \citep{khanov2024args} and PARGS \citep{rashid2024critical} perform poorly while CD-Q \citep{mudgal2023controlled} is strongest (Figures~\ref{fig:model_eval_llama},~\ref{fig:gpt_eval_llama}). This gap follows from how each scores partial sequences. ARGS uses a reward model trained on full responses to score prefixes, which is mismatched to prefix-level scoring \citep{rashid2024critical}. PARGS trains a partial reward model but assumes prefixes inherit the preference label of their full response, which need not hold. CD-Q instead trains a prefix value function estimating expected future reward, a more principled approach to token-level guidance, which explains its stronger performance.

\section{Conclusion}
\label{sec:conclusion}

In this work, we introduced \emph{Lagrangian Reward Augmentation} (LARA), a decoder-agnostic framework for safe inference-time alignment. By dualizing a KL-regularized constrained objective, LARA reduces the reward--cost trade-off to a one-dimensional calibration problem whose solution defines a safety-calibrated augmented reward. This yields a principled alternative to manual penalty tuning and can be integrated directly into existing inference-time alignment methods. For sequence-level samplers, the resulting dual variable has an exact expected-cost interpretation, while for token-level reward-guided decoders it provides a principled heuristic induced by the same constrained objective. Overall, LARA shows how safety constraints can be transferred from training-time alignment into practical decoding-time control without updating model weights. Experimentally, we find that LARA improves the helpfulness-harmlessness tradeoff, with Best-of-$N$ achieving the best performance among inference-time methods, approaching finetuning-based direct alignment baselines. Crucially, LARA achieves this without any weight updates: calibrating the dual variable requires only a small set of prompts and a fast one-dimensional search, in contrast to the substantial cost of finetuning-based alignment.

\section*{Acknowledgements}
This work has taken place in the Safe, Correct, and Aligned Learning and Robotics Lab (SCALAR) at The University of Massachusetts Amherst. SCALAR research is supported in part by the NSF (IIS-2437426) and Open Philanthropy. The computational resources for this work were provided by the University of Massachusetts Amherst’s partnership with the Unity Research Computing Platform, a multi-institutional cluster led by the University of Massachusetts and the University of Rhode Island.

Scott Niekum holds concurrent appointments as an Associate Professor at the University of Massachusetts Amherst and as an Amazon Scholar. This paper describes work performed
at the University of Massachusetts Amherst and is not associated with Amazon

\bibliography{colm2026_conference}
\bibliographystyle{colm2026_conference}

\appendix
\section*{Large Language Model Usage}

Large Language Models (LLMs) were used for grammatical editing and improving writing flow. Furthermore, LLMs were used by the authors to implement baselines, whose code was not available online, e.g., CD-Q \citep{mudgal2023controlled}, and also to create plots. The implementations were reviewed and validated by the authors. All research methodology, experimental design, data analysis, and scientific conclusions are entirely the work of the human authors.
\section{Related Work}
\label{sec:related_work}
Reinforcement Learning from Human Feedback (RLHF) is 
an idea that traces back to preference-based learning~\citep{Christiano2017DeepRL} in reinforcement learning and was brought into the LLM setting by InstructGPT~\citep{ouyang2022training}. %
In recent alignment work, this pipeline has been complemented by direct preference-optimization methods such as DPO~\citep{rafailov2023direct}, which exploit the closed-form solution of a KL-regularized reward-maximization problem to avoid performing RL.
Together, these works establish the modern alignment landscape in which RLHF-style reward-based fine-tuning remains the canonical baseline, while direct alignment methods offer cheaper but still weight-updating alternatives.

Safety has emerged as a central concern within this alignment pipeline because maximizing user preference alone does not reliably prevent toxic, misleading, or otherwise harmful generations. 
Empirical studies such as RealToxicityPrompts~\citep{gehman2020realtoxicityprompts} document toxic degeneration in neural language models, red-teaming work~\citep{ganguli2022red} systematically catalogs harmful failure modes and their scaling behavior, and broader surveys of language-model harms~\citep{weidinger2021ethical} emphasize that safety risks are social as well as technical. 
On the data side, BeaverTails~\citep{ji2023beavertails} provides a large-scale safety-alignment dataset with explicit harmfulness annotations, enabling the separation of usefulness and safety signals. 
Building on this perspective, Safe RLHF~\citep{dai2023safe} formalizes safety alignment by learning separate reward and cost models and solving a constrained optimization problem that maximizes helpfulness subject to a harmlessness budget, while SafeDPO~\citep{kim2025safedpo} develops a direct-preference alternative that incorporates safety without reverting to a full RL loop. 
The recent Safe RLHF Beyond Expectation paper~\citep{chittepu2026safe} sharpens the limitation of this expectation-based view: controlling only expected cost can miss rare but severe failures, motivating richer risk-sensitive safety criteria beyond a single expectation statistic.

Inference-time alignment aims to satisfy alignment objectives without additional fine-tuning by controlling a frozen model at decode time. 
This setting is attractive when rapid deployment, low retraining overhead, or policy modularity is required.
Best-of-N reranking~\citep{nakano2022webgptbrowserassistedquestionansweringhuman, stiennon2022learningsummarizehumanfeedback,beirami2024theoretical} samples several complete responses and selects the one with the best reward score, giving a simple sequence-level test-time alignment mechanism with theoretical guarantees. 
ARGS~\citep{khanov2024args} instead performs reward-guided search directly during generation by modifying token selection using a reward signal at decoding time. Controlled Decoding~\citep{mudgal2023controlled} learns prefix scorers that bias next-token probabilities toward desired attributes. Reward-Augmented Decoding~\citep{deng-raffel-2023-reward} likewise injects reward information into token-level generation, but is designed around a unidirectional reward model to make controlled decoding more efficient. FUDGE~\citep{yang2021fudge} is an earlier controlled-generation approach that uses future discriminators to steer generation toward target attributes without retraining the base LM. \cite{rashid2024critical} provide an important critique of tokenwise reward guidance, showing that full-sequence reward models are generally mismatched to prefix-level scoring and motivating partial-sequence training; this line of work underlies PARGS. 
FaRMA~\citep{rashid2025towards} then addresses the computational side of the same problem by reducing reward-model calls and making reward-guided decoding substantially cheaper at test time. FaRMA also uses the optimal completion score of partial sequences to guide decoding, unlike earlier methods. Our work is orthogonal to these decoder designs: rather than proposing yet another search rule, it derives a calibrated helpfulness-versus-harmfulness objective from the dual of constrained Safe RLHF and uses that calibrated score as a drop-in signal for existing inference-time methods, thereby replacing manual scalar trade-off tuning with a principled safety-constrained construction.

\section{Proof of Theorem~\ref{thm:dual-closed-form}}
\label{sec:appendix-duality}

\begin{proof}

For notational simplicity, we present the argument in the countable-response case; the general measurable-space version follows from the same Gibbs variational identity.

Fix $\lambda\ge 0$. Starting from \eqref{eq:lagrangian},
\[
\mathcal{L}(\pi,\lambda)
=
\mathbb{E}_{x\sim D_x}\!\left[
\mathbb{E}_{y\sim \pi(\cdot\mid x)} r_\lambda(x,y)
-\beta\,D_{\mathrm{KL}}\!\bigl(\pi(\cdot\mid x)\,\|\,\pi_{\mathrm{ref}}(\cdot\mid x)\bigr)
\right]
+\lambda\tau.
\]
Therefore
\begin{equation}
\label{eq:dual-start}
g(\lambda)
=
\sup_{\pi}\mathcal{L}(\pi,\lambda)
=
\lambda\tau
+
\sup_{\pi}
\mathbb{E}_{x\sim D_x}\!\left[
\mathbb{E}_{y\sim \pi(\cdot\mid x)} r_\lambda(x,y)
-\beta\,D_{\mathrm{KL}}\!\bigl(\pi(\cdot\mid x)\,\|\,\pi_{\mathrm{ref}}(\cdot\mid x)\bigr)
\right].
\end{equation}
Since the optimization over $\pi(\cdot\mid x)$ is pointwise in $x$, the supremum decouples across prompts:
\begin{equation}
\label{eq:dual-decouple}
g(\lambda)
=
\lambda\tau
+
\mathbb{E}_{x\sim D_x}
\left[
\sup_{q\ll \pi_{\mathrm{ref}}(\cdot\mid x)}
\left\{
\mathbb{E}_{y\sim q} r_\lambda(x,y)
-\beta\,D_{\mathrm{KL}}\!\bigl(q\,\|\,\pi_{\mathrm{ref}}(\cdot\mid x)\bigr)
\right\}
\right].
\end{equation}
Fix a prompt $x$, and abbreviate
\[
p_x(y):=\pi_{\mathrm{ref}}(y\mid x).
\]
For any candidate distribution $q$ on responses,
\begin{equation}
\label{eq:Fx-def}
F_x(q)
:=
\mathbb{E}_{y\sim q} r_\lambda(x,y)
-\beta\,D_{\mathrm{KL}}(q\|p_x)
=
\sum_y q(y)\,r_\lambda(x,y)
-\beta\sum_y q(y)\log\frac{q(y)}{p_x(y)}.
\end{equation}

We now solve the maximization of $F_x(q)$.

\smallskip
\noindent\textbf{First-order characterization of the optimizer.}
Introduce a multiplier $\eta$ for the normalization constraint $\sum_y q(y)=1$, and consider
\[
\mathcal{J}(q,\eta)
=
\sum_y q(y)\,r_\lambda(x,y)
-\beta\sum_y q(y)\log\frac{q(y)}{p_x(y)}
+\eta\left(\sum_y q(y)-1\right).
\]
Differentiate with respect to $q(y)$:
\[
\frac{\partial \mathcal{J}}{\partial q(y)}
=
r_\lambda(x,y)
-\beta\left(\log\frac{q(y)}{p_x(y)}+1\right)
+\eta.
\]
Setting the derivative equal to zero gives
\[
r_\lambda(x,y)
-\beta\left(\log\frac{q(y)}{p_x(y)}+1\right)
+\eta
=0,
\]
hence
\[
\log\frac{q(y)}{p_x(y)}
=
\frac{r_\lambda(x,y)+\eta-\beta}{\beta}.
\]
Exponentiating both sides yields
\[
q(y)
=
p_x(y)\exp\!\left(\frac{r_\lambda(x,y)+\eta-\beta}{\beta}\right).
\]
Thus $q(y)$ is proportional to $p_x(y)\exp(r_\lambda(x,y)/\beta)$, i.e.
\[
q(y)\propto p_x(y)\exp\!\left(\frac{r_\lambda(x,y)}{\beta}\right).
\]
Normalizing shows that the optimizer must be
\begin{equation}
\label{eq:q-star-x}
\pi_\lambda^*(y\mid x)
=
\frac{
p_x(y)\exp(r_\lambda(x,y)/\beta)
}{
Z_x(\lambda)
},
\end{equation}
where
\begin{equation}
\label{eq:partition-x}
Z_x(\lambda)
:=
\sum_y p_x(y)\exp\!\left(\frac{r_\lambda(x,y)}{\beta}\right)
=
\mathbb{E}_{y\sim \pi_{\mathrm{ref}}(\cdot\mid x)}
\exp\!\left(\frac{r_\lambda(x,y)}{\beta}\right).
\end{equation}
This is exactly \eqref{eq:gibbs-policy}.

\smallskip
\noindent\textbf{Compute the optimal value exactly.}
We now evaluate the supremum in closed form by completing the KL divergence. By \eqref{eq:q-star-x},
\[
\log \pi_\lambda^*(y\mid x)
=
\log p_x(y)+\frac{r_\lambda(x,y)}{\beta}-\log Z_x(\lambda).
\]
Therefore, for any $q\ll p_x$,
\begin{align*}
D_{\mathrm{KL}}(q\|\pi_\lambda^*(\cdot\mid x))
&=
\mathbb{E}_{y\sim q}
\left[
\log\frac{q(y)}{\pi_\lambda^*(y\mid x)}
\right] \\
&=
\mathbb{E}_{y\sim q}
\left[
\log q(y)-\log p_x(y)-\frac{r_\lambda(x,y)}{\beta}+\log Z_x(\lambda)
\right] \\
&=
\mathbb{E}_{y\sim q}
\left[
\log\frac{q(y)}{p_x(y)}
\right]
-\frac{1}{\beta}\mathbb{E}_{y\sim q}r_\lambda(x,y)
+\log Z_x(\lambda) \\
&=
D_{\mathrm{KL}}(q\|p_x)
-\frac{1}{\beta}\mathbb{E}_{y\sim q}r_\lambda(x,y)
+\log Z_x(\lambda).
\end{align*}
Rearranging gives
\begin{equation}
\label{eq:complete-kl}
\mathbb{E}_{y\sim q}r_\lambda(x,y)
-\beta D_{\mathrm{KL}}(q\|p_x)
=
\beta\log Z_x(\lambda)
-\beta D_{\mathrm{KL}}(q\|\pi_\lambda^*(\cdot\mid x)).
\end{equation}
Since KL divergence is always nonnegative,
\[
D_{\mathrm{KL}}(q\|\pi_\lambda^*(\cdot\mid x))\ge 0,
\]
and therefore \eqref{eq:complete-kl} implies
\[
\mathbb{E}_{y\sim q}r_\lambda(x,y)
-\beta D_{\mathrm{KL}}(q\|p_x)
\le
\beta\log Z_x(\lambda),
\]
with equality if and only if $q=\pi_\lambda^*(\cdot\mid x)$ almost surely. Hence
\begin{equation}
\label{eq:Fx-sup}
\sup_{q\ll p_x}
\left\{
\mathbb{E}_{y\sim q} r_\lambda(x,y)
-\beta D_{\mathrm{KL}}(q\|p_x)
\right\}
=
\beta\log Z_x(\lambda).
\end{equation}
Substituting \eqref{eq:Fx-sup} into \eqref{eq:dual-decouple} yields
\[
g(\lambda)
=
\lambda\tau
+
\mathbb{E}_{x\sim D_x}\bigl[\beta\log Z_x(\lambda)\bigr].
\]
Using \eqref{eq:partition-x}, we obtain
\[
g(\lambda)
=
\lambda\tau
+
\beta\,\mathbb{E}_{x\sim D_x}
\log
\mathbb{E}_{y\sim \pi_{\mathrm{ref}}(\cdot\mid x)}
\exp\!\left(\frac{r_\lambda(x,y)}{\beta}\right).
\]
Finally, substituting $r_\lambda(x,y)=r(x,y)-\lambda c(x,y)$ gives
\[
g(\lambda)
=
\beta\,\mathbb{E}_{x\sim D_x}
\log
\mathbb{E}_{y\sim \pi_{\mathrm{ref}}(\cdot\mid x)}
\exp\!\left(\frac{r(x,y)-\lambda c(x,y)}{\beta}\right)
+\lambda\tau,
\]
which is exactly \eqref{eq:dual-objective}. This completes the proof.
\end{proof}

\section{Proof of Lemma~\ref{lem:dual-derivatives}}
\label{sec:lemma-proof}
\begin{proof}

Define
\[
A_x(\lambda)
:=
\beta\log
\mathbb{E}_{y\sim \pi_{\mathrm{ref}}(\cdot\mid x)}
\exp\!\left(\frac{r(x,y)-\lambda c(x,y)}{\beta}\right).
\]
Then
\[
g(\lambda)=\mathbb{E}_{x\sim D_x}A_x(\lambda)+\lambda\tau.
\]
Fix $x$ and write
\[
Z_x(\lambda)
:=
\mathbb{E}_{y\sim \pi_{\mathrm{ref}}(\cdot\mid x)}
\exp\!\left(\frac{r(x,y)-\lambda c(x,y)}{\beta}\right).
\]
Then
\[
A_x(\lambda)=\beta\log Z_x(\lambda).
\]
Differentiate:
\begin{align*}
A_x'(\lambda)
&=
\beta\,\frac{Z_x'(\lambda)}{Z_x(\lambda)}.
\end{align*}
Now
\begin{align*}
Z_x'(\lambda)
&=
\mathbb{E}_{y\sim \pi_{\mathrm{ref}}(\cdot\mid x)}
\left[
\exp\!\left(\frac{r(x,y)-\lambda c(x,y)}{\beta}\right)
\cdot \left(-\frac{c(x,y)}{\beta}\right)
\right].
\end{align*}
Therefore
\begin{align*}
A_x'(\lambda)
&=
\beta\,
\frac{
\mathbb{E}_{y\sim \pi_{\mathrm{ref}}(\cdot\mid x)}
\left[
\exp\!\left(\frac{r(x,y)-\lambda c(x,y)}{\beta}\right)
\left(-\frac{c(x,y)}{\beta}\right)
\right]
}{
\mathbb{E}_{y\sim \pi_{\mathrm{ref}}(\cdot\mid x)}
\exp\!\left(\frac{r(x,y)-\lambda c(x,y)}{\beta}\right)
} \\
&=
-
\frac{
\mathbb{E}_{y\sim \pi_{\mathrm{ref}}(\cdot\mid x)}
\left[
c(x,y)\exp\!\left(\frac{r(x,y)-\lambda c(x,y)}{\beta}\right)
\right]
}{
\mathbb{E}_{y\sim \pi_{\mathrm{ref}}(\cdot\mid x)}
\exp\!\left(\frac{r(x,y)-\lambda c(x,y)}{\beta}\right)
}.
\end{align*}
By the definition of $\pi_\lambda^*$ in \eqref{eq:gibbs-policy}, the ratio above is exactly the expectation of $c(x,y)$ under $\pi_\lambda^*(\cdot\mid x)$, so
\[
A_x'(\lambda)
=
-\mathbb{E}_{y\sim \pi_\lambda^*(\cdot\mid x)} c(x,y).
\]
Taking expectation over $x$ and differentiating $\lambda\tau$ gives
\[
g'(\lambda)
=
\mathbb{E}_{x\sim D_x}A_x'(\lambda)+\tau
=
\tau
-
\mathbb{E}_{x\sim D_x}\mathbb{E}_{y\sim \pi_\lambda^*(\cdot\mid x)} c(x,y),
\]
which proves \eqref{eq:dual-derivative}.

To compute the second derivative, write
\[
w_\lambda(y\mid x)
:=
\frac{
\pi_{\mathrm{ref}}(y\mid x)\exp((r(x,y)-\lambda c(x,y))/\beta)
}{
Z_x(\lambda)
}
=
\pi_\lambda^*(y\mid x).
\]
Then
\[
A_x'(\lambda)=-\mathbb{E}_{w_\lambda}[c(x,y)].
\]
Differentiate once more. Using the standard derivative of an exponential-family expectation,
\[
\frac{d}{d\lambda}\mathbb{E}_{w_\lambda}[c(x,y)]
=
-\frac{1}{\beta}\,
\mathrm{Var}_{w_\lambda}(c(x,y)).
\]
Hence
\[
A_x''(\lambda)
=
\frac{1}{\beta}\,
\mathrm{Var}_{y\sim \pi_\lambda^*(\cdot\mid x)}(c(x,y)).
\]
Therefore
\[
g''(\lambda)
=
\mathbb{E}_{x\sim D_x}A_x''(\lambda)
=
\frac{1}{\beta}\,
\mathbb{E}_{x\sim D_x}
\mathrm{Var}_{y\sim \pi_\lambda^*(\cdot\mid x)}(c(x,y))
\ge 0.
\]
This proves \eqref{eq:dual-hessian}, and convexity follows immediately.
\end{proof}

\section{Proof of Theorem~\ref{thm:finite-sample-empirical}}
\label{app:duality-conc}
The proof combines standard Hoeffding-type concentration on a finite grid, sample-average approximation arguments for stochastic programs, and classical convex M-estimation / argmin localization under strong convexity~\citep{van2000asymptotic}.

\begin{proof}
\textbf{Empirical tilted quantities.}
For fixed $(x,y_{1:K},\lambda)$, define
\begin{equation}\label{eq:signed_proof_w_hatZ}
w_k(\lambda)
:=
\exp\!\left(\frac{r(x,y_k)-\lambda c(x,y_k)}{\beta}\right),
\qquad
\widehat Z_{x,K}(\lambda):=\frac{1}{K}\sum_{k=1}^K w_k(\lambda),
\end{equation}
and the empirical tilted weights
\begin{equation}\label{eq:signed_proof_hatp}
\widehat p_\lambda(k\mid x,y_{1:K})
:=
\frac{w_k(\lambda)}{\sum_{\ell=1}^K w_\ell(\lambda)},
\qquad k=1,\dots,K.
\end{equation}
Define the empirical tilted cost
\begin{equation}\label{eq:signed_proof_hatc}
\widehat c_{\lambda,K}(x,y_{1:K})
:=
\sum_{k=1}^K \widehat p_\lambda(k\mid x,y_{1:K})\,c(x,y_k).
\end{equation}
Also define
\begin{equation}\label{eq:signed_proof_hatphi_local}
\widehat\phi_K(x,y_{1:K},\lambda)
:=
\beta\log \widehat Z_{x,K}(\lambda)+\lambda\tau.
\end{equation}

\textbf{Derivatives of the empirical inner loss.}
Differentiating \eqref{eq:signed_proof_w_hatZ} with respect to $\lambda$ gives
\begin{equation}\label{eq:signed_proof_hatZprime}
\frac{d}{d\lambda}\widehat Z_{x,K}(\lambda)
=
-\frac{1}{\beta K}\sum_{k=1}^K c(x,y_k)w_k(\lambda).
\end{equation}
Hence, using \eqref{eq:signed_proof_hatphi_local}, \eqref{eq:signed_proof_w_hatZ}, and \eqref{eq:signed_proof_hatc},
\begin{equation}\label{eq:signed_proof_hatphi_prime}
\widehat\phi_K'(x,y_{1:K},\lambda)
=
\beta \frac{\widehat Z_{x,K}'(\lambda)}{\widehat Z_{x,K}(\lambda)}+\tau
=
-\frac{\sum_{k=1}^K c(x,y_k)w_k(\lambda)}
{\sum_{k=1}^K w_k(\lambda)}
+\tau
=
\tau-\widehat c_{\lambda,K}(x,y_{1:K}).
\end{equation}
Differentiating once more yields
\begin{equation}\label{eq:signed_proof_hatphi_second}
\widehat\phi_K''(x,y_{1:K},\lambda)
=
\frac{1}{\beta}
\operatorname{Var}_{J\sim \widehat p_\lambda(\cdot\mid x,y_{1:K})}
\!\bigl(c(x,y_J)\bigr).
\end{equation}
Since $-C\le c(x,y)\le C$, \eqref{eq:signed_proof_hatc} and \eqref{eq:signed_proof_hatphi_second} imply
\begin{equation}\label{eq:signed_proof_hatc_hatphi_bounds}
-C\le \widehat c_{\lambda,K}(x,y_{1:K})\le C,
\qquad
0\le \widehat\phi_K''(x,y_{1:K},\lambda)\le \frac{C^2}{\beta}.
\end{equation}
Combining \eqref{eq:signed_proof_hatphi_prime} with \eqref{eq:signed_proof_hatc_hatphi_bounds} and $-C\le \tau\le C$, we obtain
\begin{equation}\label{eq:signed_proof_hatphi_prime_bound}
\bigl|\widehat\phi_K'(x,y_{1:K},\lambda)\bigr|\le 2C
\qquad
\text{for all }(x,y_{1:K},\lambda)\in \mathcal{X}\times \mathcal{Y}^K\times I.
\end{equation}

\textbf{Monte Carlo population objective.}
Define
\begin{equation}\label{eq:signed_proof_gK}
g_K(\lambda)
:=
\mathbb{E}_{X\sim D_x}
\mathbb{E}_{Y_1,\dots,Y_K \stackrel{\mathrm{i.i.d.}}{\sim}\pi_{\mathrm{ref}}(\cdot\mid X)}
\bigl[\widehat\phi_K(X,Y_{1:K},\lambda)\bigr].
\end{equation}
Since $\widehat\phi_K(\cdot,\cdot,\lambda)$ is differentiable in $\lambda$ and the uniform bound \eqref{eq:signed_proof_hatphi_prime_bound} holds on $I$, dominated convergence applied to the difference quotients implies
\begin{equation}\label{eq:signed_proof_gK_prime}
g_K'(\lambda)
=
\mathbb{E}_{X\sim D_x}
\mathbb{E}_{Y_1,\dots,Y_K \stackrel{\mathrm{i.i.d.}}{\sim}\pi_{\mathrm{ref}}(\cdot\mid X)}
\bigl[\widehat\phi_K'(X,Y_{1:K},\lambda)\bigr].
\end{equation}
Also, by \eqref{eq:thm_empirical_inner_hatphi_hatg},
\begin{equation}\label{eq:signed_proof_hatg_prime}
\widehat g_{N,K}'(\lambda)
=
\frac{1}{N}\sum_{i=1}^N \widehat\phi_{i,K}'(\lambda).
\end{equation}
Finally, \eqref{eq:signed_proof_hatc_hatphi_bounds} implies that both $g_K'$ and $\widehat g_{N,K}'$ are Lipschitz on $I$ with constant
\begin{equation}\label{eq:signed_proof_L}
L:=\frac{C^2}{\beta}.
\end{equation}

\textbf{Exact tilted quantities and derivative of $g$.}
For fixed $(x,\lambda)$, define the exact tilted distribution
\begin{equation}\label{eq:signed_proof_exact_tilt}
\pi_\lambda^\star(y\mid x)
:=
\frac{
\pi_{\mathrm{ref}}(y\mid x)
\exp\!\left(\frac{r(x,y)-\lambda c(x,y)}{\beta}\right)
}{
Z_x(\lambda)
},
\end{equation}
and the corresponding tilted cost
\begin{equation}\label{eq:signed_proof_exact_tilted_cost}
\bar c_\lambda(x)
:=
\mathbb{E}_{Y\sim \pi_\lambda^\star(\cdot\mid x)}[c(x,Y)].
\end{equation}
Differentiating \eqref{eq:thm_empirical_inner_Z_phi} gives
\begin{equation}\label{eq:signed_proof_Zprime_exact}
Z_x'(\lambda)
=
\mathbb{E}_{Y\sim \pi_{\mathrm{ref}}(\cdot\mid x)}
\left[
-\frac{c(x,Y)}{\beta}
\exp\!\left(\frac{r(x,Y)-\lambda c(x,Y)}{\beta}\right)
\right],
\end{equation}
and hence
\begin{equation}\label{eq:signed_proof_phi_prime_exact}
\phi'(x,\lambda)
=
\beta \frac{Z_x'(\lambda)}{Z_x(\lambda)}+\tau
=
\tau-\bar c_\lambda(x).
\end{equation}
Differentiating once more yields
\begin{equation}\label{eq:signed_proof_phi_second_exact}
\phi''(x,\lambda)
=
\frac{1}{\beta}
\operatorname{Var}_{Y\sim \pi_\lambda^\star(\cdot\mid x)}
\!\bigl(c(x,Y)\bigr).
\end{equation}
Since $-C\le c(x,y)\le C$, \eqref{eq:signed_proof_phi_prime_exact} implies
\begin{equation}\label{eq:signed_proof_phi_prime_bound_exact}
|\phi'(x,\lambda)|\le 2C.
\end{equation}
Therefore, dominated convergence applied to the difference quotients gives
\begin{equation}\label{eq:signed_proof_g_prime_exact}
g'(\lambda)
=
\mathbb{E}_{X\sim D_x}[\phi'(X,\lambda)]
=
\tau
-
\mathbb{E}_{X\sim D_x}
\mathbb{E}_{Y\sim \pi_\lambda^\star(\cdot\mid X)}
[c(X,Y)].
\end{equation}

\textbf{Concentration on a finite grid.}
Fix an integer $M\ge 1$ and define the grid
\begin{equation}\label{eq:signed_proof_grid}
\lambda_j:=\frac{j\Lambda}{M},
\qquad j=0,1,\dots,M,
\end{equation}
with spacing
\begin{equation}\label{eq:signed_proof_grid_spacing}
\Delta:=\frac{\Lambda}{M}.
\end{equation}
For each grid point $\lambda_j$, let
\begin{equation}\label{eq:signed_proof_Uij}
U_i^{(j)}:=\widehat\phi_{i,K}'(\lambda_j).
\end{equation}
For each fixed $j$, the random variables $U_1^{(j)},\dots,U_N^{(j)}$ are i.i.d.\ across $i$, and by \eqref{eq:signed_proof_hatphi_prime_bound},
\begin{equation}\label{eq:signed_proof_Uij_bound}
-2C\le U_i^{(j)}\le 2C
\qquad\text{a.s.}
\end{equation}
Therefore, Hoeffding's inequality gives, for every $t>0$,
\begin{equation}\label{eq:signed_proof_hoeffding_grid}
\mathbb{P}\!\left(
\left|
\widehat g_{N,K}'(\lambda_j)-g_K'(\lambda_j)
\right|>t
\right)
\le
2\exp\!\left(-\frac{Nt^2}{8C^2}\right).
\end{equation}
Applying the union bound over $j=0,\dots,M$ yields
\begin{equation}\label{eq:signed_proof_union_bound_grid}
\mathbb{P}\!\left(
\max_{0\le j\le M}
\left|
\widehat g_{N,K}'(\lambda_j)-g_K'(\lambda_j)
\right|>t
\right)
\le
2(M+1)\exp\!\left(-\frac{Nt^2}{8C^2}\right).
\end{equation}
Set
\begin{equation}\label{eq:signed_proof_tNM}
t_{N,M}(\delta)
:=
2C\sqrt{\frac{2\log(2(M+1)/\delta)}{N}}.
\end{equation}
Then \eqref{eq:signed_proof_union_bound_grid} implies that, with probability at least $1-\delta$,
\begin{equation}\label{eq:signed_proof_grid_event}
\max_{0\le j\le M}
\left|
\widehat g_{N,K}'(\lambda_j)-g_K'(\lambda_j)
\right|
\le
t_{N,M}(\delta).
\end{equation}

\textbf{Uniform derivative control relative to $g_K'$.}
On the event \eqref{eq:signed_proof_grid_event}, we claim that
\begin{equation}\label{eq:signed_proof_uniform_hatg_minus_gK}
\sup_{\lambda\in I}
\left|
\widehat g_{N,K}'(\lambda)-g_K'(\lambda)
\right|
\le
t_{N,M}(\delta)+L\Delta.
\end{equation}
Fix $\lambda\in I$, and choose $j\in\{0,\dots,M-1\}$ such that $\lambda\in[\lambda_j,\lambda_{j+1}]$. Since $\widehat g_{N,K}$ is convex and differentiable, its derivative $\widehat g_{N,K}'$ is monotone nondecreasing, and therefore
\begin{equation}\label{eq:signed_proof_monotonicity_hatgprime}
\widehat g_{N,K}'(\lambda_j)
\le
\widehat g_{N,K}'(\lambda)
\le
\widehat g_{N,K}'(\lambda_{j+1}).
\end{equation}
Using \eqref{eq:signed_proof_monotonicity_hatgprime}, \eqref{eq:signed_proof_grid_event}, and the Lipschitz bound \eqref{eq:signed_proof_L}, we obtain
\begin{equation}\label{eq:signed_proof_upper_uniform}
\widehat g_{N,K}'(\lambda)-g_K'(\lambda)
\le
t_{N,M}(\delta)+L\Delta,
\end{equation}
and similarly
\begin{equation}\label{eq:signed_proof_lower_uniform}
\widehat g_{N,K}'(\lambda)-g_K'(\lambda)
\ge
-\,t_{N,M}(\delta)-L\Delta.
\end{equation}
Combining \eqref{eq:signed_proof_upper_uniform} and \eqref{eq:signed_proof_lower_uniform} yields \eqref{eq:signed_proof_uniform_hatg_minus_gK}.

\textbf{Inner Monte Carlo bias.}
Fix $x$ and $\lambda\in I$. Let
\begin{equation}\label{eq:signed_proof_UV}
V
:=
\exp\!\left(\frac{r(x,Y)-\lambda c(x,Y)}{\beta}\right),
\qquad
U:=c(x,Y)V,
\end{equation}
where
\begin{equation}\label{eq:signed_proof_Ylaw}
Y\sim \pi_{\mathrm{ref}}(\cdot\mid x).
\end{equation}
Define
\begin{equation}\label{eq:signed_proof_uv}
u:=\mathbb{E}_{Y\sim \pi_{\mathrm{ref}}(\cdot\mid x)}[U],
\qquad
v:=\mathbb{E}_{Y\sim \pi_{\mathrm{ref}}(\cdot\mid x)}[V].
\end{equation}
Then \eqref{eq:signed_proof_exact_tilted_cost} implies
\begin{equation}\label{eq:signed_proof_exact_cost_ratio}
\bar c_\lambda(x)=\frac{u}{v}.
\end{equation}
Let $Y_1,\dots,Y_K \stackrel{\mathrm{i.i.d.}}{\sim}\pi_{\mathrm{ref}}(\cdot\mid x)$ and define
\begin{equation}\label{eq:signed_proof_hatuv}
\widehat u:=\frac{1}{K}\sum_{k=1}^K U_k,
\qquad
\widehat v:=\frac{1}{K}\sum_{k=1}^K V_k,
\end{equation}
where
\begin{equation}\label{eq:signed_proof_UkVk}
U_k:=c(x,Y_k)\exp\!\left(\frac{r(x,Y_k)-\lambda c(x,Y_k)}{\beta}\right),
\qquad
V_k:=\exp\!\left(\frac{r(x,Y_k)-\lambda c(x,Y_k)}{\beta}\right).
\end{equation}
Then \eqref{eq:signed_proof_hatc} implies
\begin{equation}\label{eq:signed_proof_hatc_ratio}
\widehat c_{\lambda,K}(x,Y_{1:K})=\frac{\widehat u}{\widehat v}.
\end{equation}

Because $|r(x,y)|\le R$, $-C\le c(x,y)\le C$, and $0\le \lambda\le \Lambda$,
\begin{equation}\label{eq:signed_proof_reward_cost_bounds}
-(R+\Lambda C)\le r(x,y)-\lambda c(x,y)\le R+\Lambda C,
\end{equation}
and therefore, using \eqref{eq:thm_empirical_inner_ab},
\begin{equation}\label{eq:signed_proof_V_bounds}
a
\le
V_k
\le
b.
\end{equation}
Hence
\begin{equation}\label{eq:signed_proof_UV_bounds}
-Cb\le U_k\le Cb,
\qquad
v\ge a,
\qquad
\widehat v\ge a
\quad\text{a.s.}
\end{equation}
Using \eqref{eq:signed_proof_exact_cost_ratio}, \eqref{eq:signed_proof_hatuv}, \eqref{eq:signed_proof_hatc_ratio}, and \eqref{eq:signed_proof_UV_bounds}, we have
\begin{equation}\label{eq:signed_proof_ratio_difference}
\left|
\widehat c_{\lambda,K}(x,Y_{1:K})-\bar c_\lambda(x)
\right|
=
\left|
\frac{\widehat u}{\widehat v}-\frac{u}{v}
\right|
\le
\frac{|\widehat u-u|}{a}
+
\frac{|u|}{a^2}|\widehat v-v|.
\end{equation}
Since $|u|\le Cb$, \eqref{eq:signed_proof_ratio_difference} implies
\begin{equation}\label{eq:signed_proof_ratio_difference_simplified}
\left|
\widehat c_{\lambda,K}(x,Y_{1:K})-\bar c_\lambda(x)
\right|
\le
\frac{|\widehat u-u|}{a}
+
\frac{Cb}{a^2}|\widehat v-v|.
\end{equation}
Taking conditional expectation over $Y_1,\dots,Y_K \stackrel{\mathrm{i.i.d.}}{\sim}\pi_{\mathrm{ref}}(\cdot\mid x)$, we obtain
\begin{equation}\label{eq:signed_proof_hat_u_deviation}
\mathbb{E}\!\left[
|\widehat u-u|
\,\middle|\,
X=x
\right]
\le
\sqrt{\operatorname{Var}(\widehat u\mid X=x)}
\le
\frac{Cb}{\sqrt{K}},
\end{equation}
because $U_k\in[-Cb,Cb]$, and
\begin{equation}\label{eq:signed_proof_hat_v_deviation}
\mathbb{E}\!\left[
|\widehat v-v|
\,\middle|\,
X=x
\right]
\le
\sqrt{\operatorname{Var}(\widehat v\mid X=x)}
\le
\frac{b}{2\sqrt{K}},
\end{equation}
because $V_k\in[a,b]\subseteq[0,b]$.
Therefore, using \eqref{eq:signed_proof_ratio_difference_simplified}, \eqref{eq:signed_proof_hat_u_deviation}, and \eqref{eq:signed_proof_hat_v_deviation},
\begin{equation}\label{eq:signed_proof_hatc_bias_bound}
\mathbb{E}\!\left[
\left|
\widehat c_{\lambda,K}(x,Y_{1:K})-\bar c_\lambda(x)
\right|
\,\middle|\,
X=x
\right]
\le
\frac{Cb}{a\sqrt{K}}
+
\frac{Cb^2}{2a^2\sqrt{K}}
=
B_K.
\end{equation}
This bound is uniform in $x$ and $\lambda\in I$.

\textbf{Uniform comparison with the exact derivative.}
Using \eqref{eq:signed_proof_gK_prime}, \eqref{eq:signed_proof_hatphi_prime}, and \eqref{eq:signed_proof_g_prime_exact}, we have
\begin{equation}\label{eq:signed_proof_gKprime_exact}
g_K'(\lambda)
=
\tau
-
\mathbb{E}_{X\sim D_x}
\mathbb{E}_{Y_1,\dots,Y_K \stackrel{\mathrm{i.i.d.}}{\sim}\pi_{\mathrm{ref}}(\cdot\mid X)}
\bigl[
\widehat c_{\lambda,K}(X,Y_{1:K})
\bigr].
\end{equation}
Hence \eqref{eq:signed_proof_hatc_bias_bound} implies that, for every $\lambda\in I$,
\begin{equation}\label{eq:signed_proof_gK_minus_gprime}
|g_K'(\lambda)-g'(\lambda)|\le B_K.
\end{equation}
Combining \eqref{eq:signed_proof_uniform_hatg_minus_gK} with \eqref{eq:signed_proof_gK_minus_gprime}, we conclude that on the event \eqref{eq:signed_proof_grid_event},
\begin{equation}\label{eq:signed_proof_uniform_hatg_minus_g}
\sup_{\lambda\in I}
\left|
\widehat g_{N,K}'(\lambda)-g'(\lambda)
\right|
\le
t_{N,M}(\delta)+L\Delta+B_K.
\end{equation}
Set
\begin{equation}\label{eq:signed_proof_epsilon_temp}
\varepsilon:=t_{N,M}(\delta)+L\Delta+B_K.
\end{equation}

\textbf{Localization of the empirical minimizer.}
Since $g$ has a unique minimizer $\lambda^\star\in(0,\Lambda)$, we have
\begin{equation}\label{eq:signed_proof_gprime_at_minimizer}
g'(\lambda^\star)=0.
\end{equation}
From Lemma~\ref{lem:dual-derivatives}, for every $\lambda\ge \lambda^\star$,
\begin{equation}\label{eq:signed_proof_right_growth}
g'(\lambda)
=
\int_{\lambda^\star}^{\lambda} g''(u)\,du
\ge
\mu(\lambda-\lambda^\star),
\end{equation}
while for every $\lambda\le \lambda^\star$,
\begin{equation}\label{eq:signed_proof_left_growth}
g'(\lambda)
=
-\int_{\lambda}^{\lambda^\star} g''(u)\,du
\le
-\mu(\lambda^\star-\lambda).
\end{equation}
Therefore, if $\lambda\ge \lambda^\star+\varepsilon/\mu$, then \eqref{eq:signed_proof_right_growth} implies $g'(\lambda)\ge \varepsilon$, and hence, by \eqref{eq:signed_proof_uniform_hatg_minus_g},
\begin{equation}\label{eq:signed_proof_hatgprime_nonnegative}
\widehat g_{N,K}'(\lambda)
\ge
g'(\lambda)-\varepsilon
\ge 0.
\end{equation}
Likewise, if $\lambda\le \lambda^\star-\varepsilon/\mu$, then \eqref{eq:signed_proof_left_growth} implies $g'(\lambda)\le -\varepsilon$, and hence
\begin{equation}\label{eq:signed_proof_hatgprime_nonpositive}
\widehat g_{N,K}'(\lambda)
\le
g'(\lambda)+\varepsilon
\le 0.
\end{equation}
If $\varepsilon<\mu m$, where $m$ is defined in \eqref{eq:thm_empirical_inner_m}, then the interval
\begin{equation}\label{eq:signed_proof_localization_interval}
\left[
\lambda^\star-\frac{\varepsilon}{\mu},
\lambda^\star+\frac{\varepsilon}{\mu}
\right]
\end{equation}
lies strictly inside $I$. By convexity of $\widehat g_{N,K}$ together with \eqref{eq:signed_proof_hatgprime_nonnegative} and \eqref{eq:signed_proof_hatgprime_nonpositive}, every minimizer of $\widehat g_{N,K}$ over $I$ must lie in the interval \eqref{eq:signed_proof_localization_interval}. Hence
\begin{equation}\label{eq:signed_proof_hatlambda_bound_eps}
\bigl|\widehat\lambda_{N,K}-\lambda^\star\bigr|
\le
\frac{\varepsilon}{\mu}.
\end{equation}

\textbf{Final choice of grid size.}
Choose $M=N$. Then \eqref{eq:signed_proof_grid_spacing} and \eqref{eq:signed_proof_tNM} become
\begin{equation}\label{eq:signed_proof_grid_choice}
\Delta=\frac{\Lambda}{N},
\qquad
t_{N,N}(\delta)
=
2C\sqrt{\frac{2\log(2(N+1)/\delta)}{N}}.
\end{equation}
Substituting \eqref{eq:signed_proof_grid_choice}, \eqref{eq:signed_proof_L}, and \eqref{eq:thm_empirical_inner_BK} into \eqref{eq:signed_proof_epsilon_temp}, we obtain
\begin{equation}\label{eq:signed_proof_epsilon_equals_final}
\varepsilon
=
2C\sqrt{\frac{2\log(2(N+1)/\delta)}{N}}
+
\frac{C^2\Lambda}{\beta N}
+
B_K
=
\varepsilon_{N,K}(\delta).
\end{equation}
Combining \eqref{eq:signed_proof_hatlambda_bound_eps} with \eqref{eq:signed_proof_epsilon_equals_final}, and using the probability of the event \eqref{eq:signed_proof_grid_event}, yields
\begin{equation}\label{eq:signed_proof_final_probability}
\mathbb{P}\!\left(
\bigl|\widehat\lambda_{N,K}-\lambda^\star\bigr|
\le
\frac{\varepsilon_{N,K}(\delta)}{\mu}
\right)
\ge 1-\delta.
\end{equation}
This is exactly \eqref{eq:thm_empirical_inner_conclusion}.
\end{proof}

\section{Implementation Details}
\label{ap:implementation-details}
\subsection{Baselines}

We build on top of the Safe-RLHF codebase\footnote{\url{https://github.com/pku-alignment/safe-rlhf}}. 
For Safe-RLHF \citep{dai2023safe}, we use the hyperparameters reported in their paper. We set the cost threshold $\tau = -4$ for Llama and $\tau = 0.8$ for Qwen. This is because the Llama SFT model already achieves a low expected cost, making a stricter threshold appropriate, whereas the Qwen SFT model requires a more permissive threshold. To show that this choice is principled rather than arbitrary, we sweep $\tau$ for Llama-3.2-3B (Table~\ref{tab:tau_sweep}). $\lambda^*$ increases monotonically as $\tau$ tightens, $\mathbb{E}[c]$ tracks $\tau$ closely, and the constraint is satisfied at every feasible budget.
SafeDPO \citep{kim2025safedpo} does not have a publicly available implementation; we 
implement their approach by creating the data transformation script that augments the 
preference data as described in their paper, and use the DPO implementation from the 
Safe-RLHF codebase with this transformed data, using their reported hyperparameters.

\begin{table}[t]
    \centering
    \begin{tabular}{ccccc}
        \toprule
        $\tau$ & $\lambda^*$ & $\mathbb{E}[c]$ & slack ($\mathbb{E}[c] - \tau$) \\
        \midrule
        $-2.00$ & $0.000$ & $-3.503$ & $-1.503$ \\
        $-3.00$ & $0.000$ & $-3.503$ & $-0.503$ \\
        $-3.50$ & $0.066$ & $-3.641$ & $-0.141$ \\
        $-3.75$ & $0.266$ & $-3.887$ & $-0.137$ \\
        $-4.00$ & $0.605$ & $-4.131$ & $-0.131$ \\
        $-4.25$ & $1.662$ & $-4.376$ & $-0.126$ \\
        $-4.50$ & $17.373$ & $-4.519$ & $-0.019$ \\
        $-5.00$ & \multicolumn{3}{c}{INFEASIBLE} \\
        \bottomrule
    \end{tabular}
    \caption{$\tau$ sweep for Llama-3.2-3B for the BoN-20 inference-time policy. $\lambda^*$ increases monotonically as $\tau$ tightens, and the constraint is satisfied at every feasible $\tau$ (negative slack). The sharp increase in $\lambda^*$ at $\tau = -4.50$ reflects the constraint approaching the boundary of feasibility; at $\tau = -5.00$ it becomes infeasible for the base model.}
    \label{tab:tau_sweep}
\end{table}

For the Reward Guided Text Generation (RGTG) methods, we use the publicly available 
PARGS codebase\footnote{\url{https://github.com/ahmadrash/PARGS}} to create the 
preference dataset over partial sequences, and train the partial reward model using 
the reward model trainer provided in the Safe-RLHF codebase. CD-Q \citep{mudgal2023controlled} 
does not have a publicly available implementation; we implement the value function 
trainer from scratch. To improve efficiency, we model the value function as a causal language model, predicting 
the values of all possible next tokens in the vocabulary in a single forward pass, thereby 
speeding up both training and inference. This follows the same strategy used in 
FaRMA~\citep{rashid2025towards} to accelerate decoding.

Safe-RLHF was trained on 4 NVIDIA A100 80GB GPUs. SafeDPO, the reward and cost models, 
partial reward and cost models, and reward and cost value functions were all trained on 
2 NVIDIA A100 80GB GPUs.

\subsection{LARA}

We use a small partition of the BeaverTails \citep{ji2023beavertails} training dataset as the calibration dataset for estimating the dual parameter $\lambda^*$, consisting of 2,000 prompts with 20 generated responses per prompt. Responses were generated from the base policy using vLLM \citep{kwon2023efficientmemorymanagementlarge}. We set the cost threshold $\tau = -4$ for Llama and $\tau = 0.8$ for Qwen, the same 
values used for the baselines. As noted in Section~\ref{sec:model}, to account for potential distribution mismatch between the 
calibration and deployment datasets, we adopt the principle of pessimism. Specifically, 
we compute the empirical distribution of $\lambda^*$ via bootstrapping and use the 97.5th 
percentile (the upper end of the 95\% confidence interval) as our estimate of $\lambda^*$. 
We generate 10,000 bootstrap samples of $\lambda^*$, which completes in approximately 
2 minutes owing to the efficiency of binary search.

\subsection{Inference}

We use vLLM \citep{kwon2023efficientmemorymanagementlarge} to generate completions for Best-of-$N$. For ARGS and PARGS, we use the 
decoding code available in the PARGS codebase.
For CD-Q, we implement our own inference script, which achieves faster decoding than 
ARGS and PARGS owing to the causal language model structure of the value function, 
as described above. 

During inference, we use top-$k$ sampling with $k=50$ and temperature $1.2$. For 
Best-of-$N$, we use $N=20$. We note that it is important to choose $N$ to be at least 
as large as the number of responses per prompt used during calibration. A smaller $N$ 
would result in an overly aggressive $\lambda^*$, as BoN would not explore as far from 
the base policy as was observed during calibration. For RGTG methods (ARGS, PARGS, CD-Q), 
we consider the top $k=50$ tokens at each decoding step with weighting $w=2$.

\section{Calibration at Larger Model Scale}\label{ap:model-scale}
To confirm LARA's calibration is not specific to 3B models, we reproduce Figure~\ref{fig:bon_constraint} for Llama-3.1-8B in Figure~\ref{fig:bon_8b}. The same pattern holds: $\lambda^*$ satisfying the safety constraint.

\begin{figure}[t]
    \centering
    \includegraphics[width=0.6\textwidth]{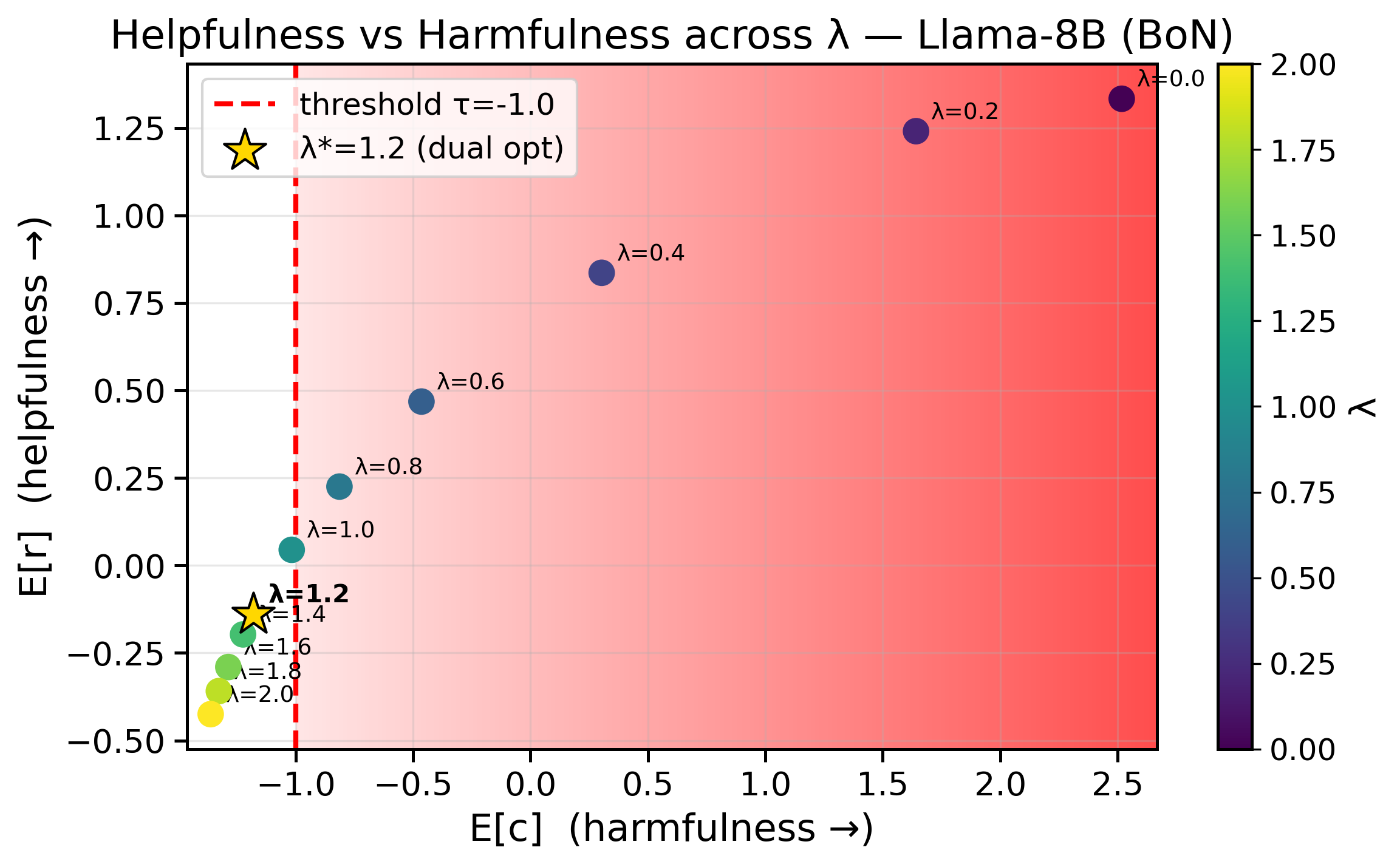}
    \caption{Expected helpfulness $\mathbb{E}[r]$ and expected cost $\mathbb{E}[c]$ of the BoN policy vs.\ $\lambda$ for Llama-3.1-8B. Dashed line indicates threshold $\tau$.}
    \label{fig:bon_8b}
\end{figure}

\section{Additional Evaluation Results}\label{ap:additional-results}

We provide the evaluation results for Qwen in Figure~\ref{fig:eval_qwen}.

\begin{figure}[t]
    \centering
    \begin{subfigure}[b]{0.48\textwidth}
        \centering
        \includegraphics[width=\textwidth]{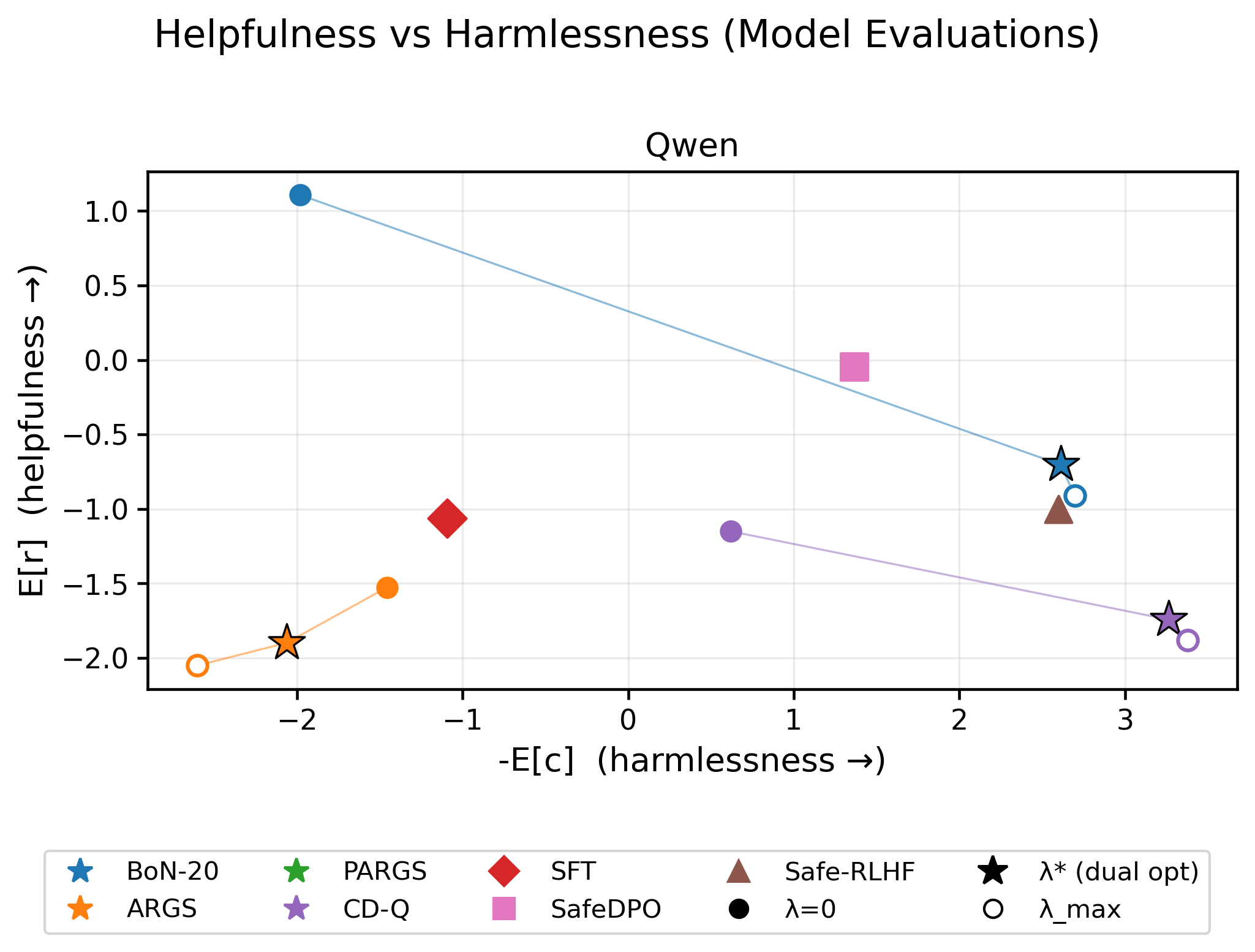}
        \caption{Model Evaluation}
        \label{fig:model_eval_qwen}
    \end{subfigure}
    \hfill
    \begin{subfigure}[b]{0.48\textwidth}
        \centering
        \includegraphics[width=\textwidth]{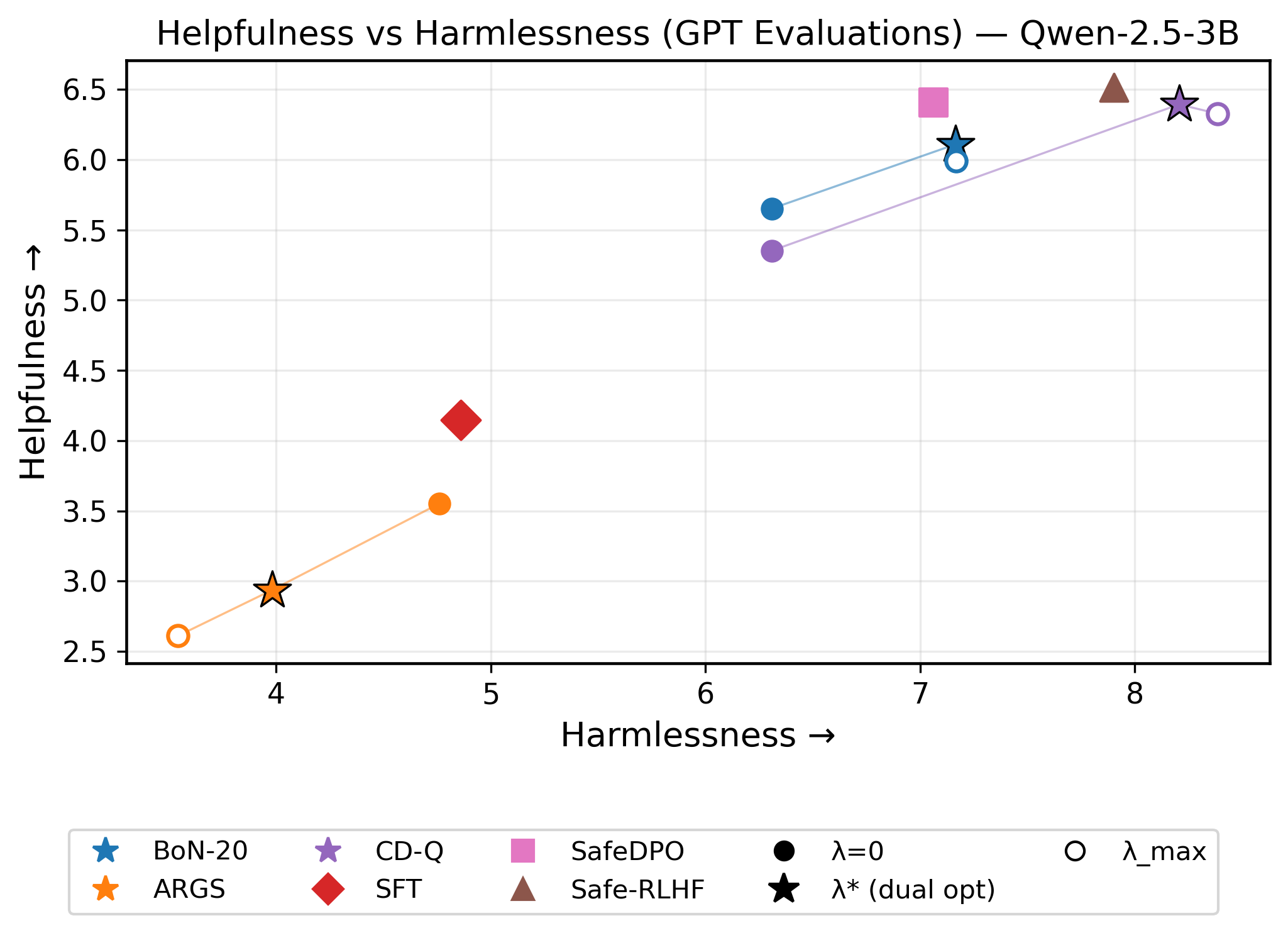}
        \caption{GPT Evaluation}
        \label{fig:gpt_eval_qwen}
    \end{subfigure}
    \caption{Helpfulness vs Harmfulness for Qwen2.5-3b under both model evaluation (left) and GPT evaluation (right), on s set of 256 prompts from the BeaverTails \citep{ji2023beavertails} test set. We observed that PARGS performed similarly to ARGS for Llama, and hence omit it from the figure. $\lambda^*$ for Qwen is $2.0$ and $\lambda_{max}$ is chosen to be $3.0$.}
    \label{fig:eval_qwen}
\end{figure}

\end{document}